\documentclass[manuscript,screen]{acmart}
\AtBeginDocument{%
  \providecommand\BibTeX{{%
    Bib\TeX}}}

\usepackage{algorithm}
\usepackage{algorithmic}
\usepackage{soul}
\usepackage{bm}
\usepackage{wrapfig}
\usepackage{array}
\newcolumntype{H}{>{\setbox0=\hbox\bgroup}c<{\egroup}@{}}

\usepackage{mathtools}
\usepackage{multirow}
\usepackage{graphicx}

\AtBeginDocument{%
  \providecommand\BibTeX{{%
    \normalfont B\kern-0.5em{\scshape i\kern-0.25em b}\kern-0.8em\TeX}}}

\newtheorem{definition}{Definition}

\newcommand*{\mline}[1]{%
\begingroup
    \renewcommand*{\arraystretch}{1.9}%
   \begin{tabular}[c]{@{}>{\centering\arraybackslash}p{8cm}@{}}#1\end{tabular}%
  \endgroup
}

\setcopyright{acmlicensed}
\copyrightyear{2025}
\acmYear{2025}
\acmDOI{XXXXXXX.XXXXXXX}
\acmConference[FAccT '25]{ACM Conference on Fairness, Accountability, and Transparency}{June 23-26,
  2025}{Athens, Greece}

\acmISBN{978-1-4503-XXXX-X/2018/06}

\begin{document}

\title{Achieving Distributive Justice in Federated Learning via Uncertainty Quantification}

\author{Alycia N. Carey}
\email{ancarey@uark.edu}
\affiliation{%
  \institution{University of Arkansas}
  \streetaddress{227 N Harmon Ave}
  \city{Fayetteville}
  \state{AR}
  \country{USA}
  \postcode{72762}
}
\author{Xintao Wu}
\email{xintaowu@uark.edu}
\affiliation{%
  \institution{University of Arkansas}
  \streetaddress{227 N Harmon Ave}
  \city{Fayetteville}
  \state{AR}
  \country{USA}
  \postcode{72762}
}

\renewcommand{\shortauthors}{Carey and Wu}

\begin{abstract}
Client-level fairness metrics for federated learning are used to ensure that all clients in a federation either: a) have similar final performance on their local data distributions (i.e., \textit{client parity}), or b) obtain final performance on their local data distributions relative to their contribution to the federated learning process (i.e., \textit{contribution fairness}). While a handful of works that propose either client-parity or contribution-based fairness metrics ground their definitions and decisions in social theories of equality -- such as distributive justice -- most works arbitrarily choose what notion of fairness to align with which makes it difficult for practitioners to choose which fairness metric aligns best with their fairness ethics. In this work, we propose \textbf{\textit{UDJ-FL}} (\textit{Uncertainty-based Distributive Justice for Federated Learning}), a flexible federated learning framework that can achieve multiple distributive justice-based client-level fairness metrics. Namely, by utilizing techniques inspired by fair resource allocation, in conjunction with performing aleatoric uncertainty-based client weighing, our UDJ-FL framework is able to achieve egalitarian, utilitarian, Rawls' difference principle, or desert-based client-level fairness. We empirically show the ability of UDJ-FL to achieve all four defined distributive justice-based client-level fairness metrics in addition to providing fairness equivalent to (or surpassing) other popular fair federated learning works. Further, we provide justification for why aleatoric uncertainty weighing is necessary to the construction of our UDJ-FL framework as well as derive theoretical guarantees for the generalization bounds of UDJ-FL. Our code is publicly available at \url{https://github.com/alycia-noel/UDJ-FL}.
\end{abstract}




\maketitle

\section{Introduction}
\label{sec:intro}
Federated learning (FL) is a machine learning paradigm that facilitates the joint training of a machine learning model by multiple parties (e.g., mobile devices, organizations, or individuals) under the organization of a central server without the parties explicitly having to share their private local data \cite{mcmahan2017communication}. Due to its potential for solving challenges in domains such as IoT \cite{nguyen2021federated} and healthcare \cite{xu2021federated}, federated learning research has received significant interest especially in the areas of efficiency, privacy, and fairness \cite{kairouz2021advances}. There are two main types of fairness explored in federated learning: client-level fairness and demographic fairness. Client-level fairness criteria are concerned with fairly distributing the performance of the global model across the clients whereas demographic fairness criteria are concerned with ensuring that traditional machine learning fairness definitions such as demographic parity or equalized odds. In this work, we focus our attention on client-level fairness.

As with traditional machine learning, multiple definitions of client-level fairness for federated learning have been proposed. Some works say that all clients should achieve similar final accuracies (often referred to as \textit{client parity}) \cite{li2019fair}. On the other hand, some works say that it is unfair for clients with low quality data to achieve the same final accuracy as clients with high-quality data \cite{donahue2023fairness,chu2022focus,ray2022fairness,lyu2020collaborative} and propose that clients should receive final performance relative to their contribution to the overall federated learning process (termed \textit{contribution fairness}). Most of these proposed works, however, only vaguely discuss the philosophical underpinnings of their chosen fairness definition, which makes it difficult for practitioners to choose the fairness metric that aligns best with their fairness ethics. Additionally, most of the proposed methods are nonflexible, and only provide solutions for one, or infrequently two, client-level fairness definitions. In this work, in addition to grounding all of our proposed client-level fairness definitions in the social psychology theory of distributive justice, we provide a flexible framework that allows practitioners to switch between the fairness definitions without requiring them to significantly alter the federated training routine. 

Distributing final model performance in relation to some chosen fairness criteria is perfectly encapsulated by the framework of distributive justice from social psychology. In this work, we focus on four major thought theories of distributive justice that align well with the federated learning setting: \textit{strict egalitarianism}, \textit{desert}, \textit{Rawls' difference principle}, and \textit{utilitarianism}. Strict egalitarianism is the concept of fairness where every person in a society receives the same level of material goods or services \cite{sep-justice-distributive}. In federated learning, we relax \textit{strict egalitarianism} to simple \textit{egalitarianism} and it is often seen as models that aim to achieve parity in the final performance of each client\footnote{This relaxation is due to it being infeasible to ensure \textit{equal} final model performance across the clients as each client has differing local data distributions and the probabilistic nature of machine learning.}. Desert-based fairness is the concept that material goods and services should be distributed according to the effort or contribution one gives to a defined goal in society \cite{sep-justice-distributive}. For federated learning, desert-based fairness could be achieved by ensuring higher performance for clients who contribute higher-quality data or those that participate in more rounds of training. Rawls' difference principle states that the only allowable difference in distributing material goods and services are those that help the least advanced in society \cite{sep-justice-distributive}. This can be seen in the federated setting as ensuring good performance for clients with low quality data or poor performing local models. Finally, utilitarianism requires that goods and services are distributed to maximize overall well being in society \cite{sep-justice-distributive}, and this can be seen in the federated learning setting as maximizing the overall average performance. While theoretically simple to understand, implementing distributive justice within federated learning is a non-trivial task. And even though the federated learning research community has begun to utilize techniques such as resource allocation (which can be neatly tied to distributive justice \cite{lan2010axiomatic}) to ensure fair distributions of resources and/or final model performance \cite{li2019fair,zhang2022proportional}, it is still difficult to determine which client is the least advantaged or how to define client contribution. In this work, we show that a client's aleatoric uncertainty -- the uncertainty relating to their local data distribution -- can be used to both specify the advantage level of the client and as a pragmatic value of client contribution. 

In this work, we propose \textbf{\textit{UDJ-FL}} (\textit{Uncertainty-based Distributive Justice for Federated Learning}), a novel federated learning objective that is grounded in distributive justice and utilizes uncertainty quantification-based client weighing. Borrowing ideas from the axiomatic approach to resource allocation \cite{lan2010axiomatic}, as well as techniques for measuring the aleatoric uncertainty of a client's local dataset \cite{mukhoti2023deep}, we construct UDJ-FL such that it can achieve all four notions of distributive justice simply by altering the chosen hyperparameters. In addition to showing how the four distributive justice principles of egalitarianism, utilitarianism, desert, and Rawls' difference principle can be recovered from our UDJ-FL objective, we empirically show that UDJ-FL is able to match (and in many cases surpass) the fairness obtained by popular client-level fair federated learning works. Throughout this work we make commentary on what fairness theory could make sense in certain settings, but we ultimately leave it up to practitioners to choose the theory that best aligns with the required fairness ethics of their specific settings. We simply provide our UDJ-FL framework to give practitioners a principled approach for ensuring distributive justice-based fairness in federated learning.

\section{Related Works}
\label{sec:related}
\begin{wraptable}{r}{0.3\textwidth}
\footnotesize
\caption{\label{tab:references}Popular client-level fairness metrics for federated learning divided along the four axis of distributive justice. E: egalitarian. U: utilitarian. R: Rawls' difference principle. D: desert. \checkmark: method achieves distributive justice theory.}
\begin{tabular}{c|cccc|}
\cline{2-5}
 & \multicolumn{4}{c|}{Theory} \\ \hline
\multicolumn{1}{|c|}{Method} & \multicolumn{1}{c|}{E} & \multicolumn{1}{c|}{U} & \multicolumn{1}{c|}{R} & D \\ \hline
\multicolumn{1}{|c|}{TERM \cite{li2020tilted}} & \checkmark &  &  &  \\
\multicolumn{1}{|c|}{FedFair$^3$ \cite{javaherian2024fedfair}} & \checkmark &  &  &  \\
\multicolumn{1}{|c|}{AdaFed \cite{hamidi2024adafed}} & \checkmark &  &  &  \\
\multicolumn{1}{|c|}{FedFV \cite{wang2021federated}} & \checkmark &  &  &  \\
\multicolumn{1}{|c|}{mFairFL \cite{su2024multi}} & \checkmark &  &  &  \\
\multicolumn{1}{|c|}{FedMGDA+ \cite{hu2020fedmgda+}} & \checkmark &  &  &  \\
\multicolumn{1}{|c|}{AdaFedAdam \cite{ju2024accelerating}} & \checkmark &  &  &  \\
\multicolumn{1}{|c|}{EFFL \cite{gao2023effl}} & \checkmark &  &  &  \\
\multicolumn{1}{|c|}{FedFa \cite{wei2020fairness}} & \checkmark &  &  &  \\
\multicolumn{1}{|c|}{E2FL \cite{mozaffari2022e2fl}} & \checkmark &  &  &  \\
\multicolumn{1}{|c|}{FedEBA+ \cite{wang2023fedeba+}} & \checkmark & \checkmark &  &  \\
\multicolumn{1}{|c|}{FedAvg \cite{mcmahan2017communication}} &  & \checkmark &  &  \\
\multicolumn{1}{|c|}{$q$-FFL \cite{li2019fair}} &  & \checkmark & \checkmark &  \\
\multicolumn{1}{|c|}{PropFair \cite{zhang2022proportional}} &  & \checkmark & \checkmark &  \\
\multicolumn{1}{|c|}{AFL \cite{mohri2019agnostic}} &  &  & \checkmark &  \\
\multicolumn{1}{|c|}{CoreFed \cite{ray2022fairness}} &  &  &  & \checkmark \\
\multicolumn{1}{|c|}{FOCUS \cite{chu2022focus}} &  &  &  & \checkmark \\
\multicolumn{1}{|c|}{CFFL \cite{lyu2020collaborative}} &  &  &  & \checkmark \\
\multicolumn{1}{|c|}{\textbf{UDJ-FL (ours)}} & \checkmark & \checkmark & \checkmark & \checkmark \\ \hline
\end{tabular}
\end{wraptable}
\paragraph{\textbf{Client-Level Fairness in Federated Learning}}

In Table \ref{tab:references} we categorize popular client-level fairness frameworks for federated learning along the four distributive justice categories we consider in this work. While some of the works explicitly mention which category they belong to (e.g., \cite{li2019fair}), we classify the others according to the fairness achieved by the proposed solutions. The majority of client-level fairness solutions proposed for federated learning align with the idea of egalitarianism and aim to ensure similar performance by the final global model on all clients in the federation \cite{mozaffari2022e2fl, su2024multi, hamidi2024adafed, wang2021federated, wang2023fedeba+, gao2023effl, wei2020fairness, javaherian2024fedfair,ju2024accelerating}. One of the first works along this direction was \cite{wang2021federated}  which identified that conflicting gradient updates with differences in the magnitudes are one cause of unfairness in federated learning and are a large contribution to why the federated learning process favors some clients over others. In this work, the authors proposed FedFV to mitigate potential conflicts among the clients' gradients before averaging. Several other works, including \cite{su2024multi, hamidi2024adafed, wang2023fedeba+}, built upon the ideas presented in \cite{wang2021federated} and have proposed other gradient alignment methods to obtain egalitarian fairness. Other approaches to egalitarian fairness that have been proposed include \cite{gao2023effl,ju2024accelerating} which used multi-objective optimization and \cite{wei2020fairness, javaherian2024fedfair} which proposed unique client weighing schemes. Here, we note that \cite{li2019fair} classified their $q$-FFL algorithm as a trade-off between egalitarianism and utilitarianism. However, we believe it is more accurate to classify $q$-FFL as a trade-off between utilitarianism and Rawls' difference principle. 

Only a handful of works have proposed solutions for client-level fairness from a non-egalitarian lens -- other than utilitarianism, which is satisfied by federated learning optimization works similar to traditional FedAvg \cite{mcmahan2017communication}. The main works that considered Rawls' DP (difference principle) fairness (or a trade-off between utilitarianism and Rawls' DP) are \cite{mohri2019agnostic, li2019fair, zhang2022proportional}. In \cite{mohri2019agnostic}, the authors proposed Agnostic Federated Learning (AFL) which optimizes the global model for any target distribution (formed by a mixture of the client distributions) instead of optimizing for the uniform distribution which is the standard in federated learning. To do so, the authors proposed a $\min\max$ learning scheme that achieves ``good intent'' fairness\footnote{Here, ``good intent'' fairness can be seen as Rawls' DP since it aims to improve the least advantaged client by focusing on the client with the highest loss.}. In \cite{li2019fair}, the authors proposed $q$-FFL which is based upon the resource allocation method of $\alpha$-fairness. $q$-FFL minimizes an aggregate reweighted loss parameterized by the variable $q$ such that clients with higher local losses are given higher weight. As $q\to\infty$, the authors note that $q$-FFL recovers AFL while as $q\to 0$ it recovers standard FedAvg which is one of the main arguments for why we classify it as achieving a trade-off between utilitarianism and Rawls' DP rather than egalitarianism. In \cite{zhang2022proportional}, the authors proposed PropFair which is based on proportional fairness to ensure fairness in the relative change of each client's performance. Finally, there are a few works proposed for achieving desert-based client-level fairness \cite{lyu2020collaborative,ray2022fairness,chu2022focus} and they all utilize some notion of client clustering to ensure proper distribution of model performance based on client contribution to the overall federated learning process. We note that while each of these proposed methods aim to achieve a single (or a trade-off between two different) client-level fairness metrics, and often use na{\"i}ve weighing approaches based on number of rounds participated or dataset size, in our work, we propose a general framework that is able to achieve all four distributive justice fairness principles using the principled approach of weighing clients based on their aleatoric uncertainty score. 

\paragraph{\textbf{Uncertainty}}
A core part of our proposed UDJ-FL framework is the use of uncertainty quantification to generate the client weights. Similar to general machine learning, uncertainty quantification has garnered much attention in the federated setting. However, as opposed to using uncertainty quantification values to achieve a notion of fairness like we do in UDJ-FL, most works focus on simply enabling uncertainty quantification in the federated setting or use uncertainty quantification values to optimize the federated learning process. For instance, in \cite{chen2022self}, the authors developed a personalized federated learning method where each client can automatically balance the training of their local model with the global model based on the calculated inter-client and intra-client uncertainty. Further, in \cite{bhatt2024federated} the authors presented an improved method of estimating uncertainty values in federated learning. We also note that in the centralized machine learning setting, the connection between uncertainty and fairness has received some interest. For example, in \cite{kuzucu2023uncertainty} the authors showed that machine learning models can be fair in regard to point-based fairness measures, but still be biased against a certain demographic group in terms of prediction uncertainty. Additionally, in \cite{tahir2023fairness} the authors showed that aleatoric uncertainty can be used to improve the fairness-utility trade-off. To our knowledge, we are the first to study the use of uncertainty quantification measures as weights for achieving distributive justice-based fairness in federated learning.

\section{Preliminary}
\label{sec:prelim}

\subsection{Federated Learning}
Federated learning is a machine learning setting where multiple clients (e.g., mobile devices, whole organizations, or individuals) collaboratively train a machine learning model under the orchestration of a central server, while keeping the training data decentralized. The most popular federated learning algorithm is Federated Averaging (FedAvg) which was proposed by McMahan et al. in 2016 \cite{mcmahan2017communication}. FedAvg aims to solve the following objective: 
\begin{equation}
    \min_{\theta} h(\theta) = \sum_{i=1}^N\frac{n_i}{\sum_{j=1}^Nn_j}H_i(\theta)
\end{equation}
where $N$ is the number of clients, $n_i$ is the number of data points held by client $i$, and $H_i(\theta) = \frac{1}{n_i}\sum_{j=1}^{n_i}\ell_j(\theta)$ is the empirical risk of client $i$. Na{\"i}vely implementing FedAvg can lead to client-level unfairness as it places more importance on clients that have larger local datasets. And, in cases where the largest clients also have low quality data (i.e., high aleatoric uncertainty), the performance of the federated model can be sub-optimal. We discuss this idea further in Appendix \ref{app:experimental}.

\subsection{Distributive Justice}
The principles of distributive justice can be thought of as providing moral guidance for the political processes and structures that affect the distribution of benefits and burdens in societies \cite{sep-justice-distributive}. Many theories of distributive justice have been proposed which often contrast each other in terms of how goods should be distributed and who in society should benefit\footnote{We note that while UDJ-FL is able to achieve all four distributive justice theories by changing the hyperparameter settings, it can only achieve one distributive justice theory at a time. I.e., UDJ-FL cannot achieve two (or more) distributive justice ideals concurrently.}. In this work, we focus on four theories of distributive justice due to their inherent connection with federated learning: \textit{strict egalitarianism, utilitarianism, Rawls' difference principle}, and \textit{desert}-based fairness. For brevity, we list the main definitions of these principles in Appendix \ref{app:dj-definitions}, and refer interested readers to \cite{sep-justice-distributive,rawls1971atheory} for an in-depth discussion on each four of these theories. In this section, we instead provide a case study of a federation of hospitals (extended from \cite{donahue2023fairness}) to show how these four principles can arise in federated learning. 

\textit{Consider three hospitals, $A$, $B$ and $C$ that exist within the same hospital network (e.g., Veterans Health Administration in the United States). The hospital network wants each hospital to be able to accurately classify brain tumor types within MR images and therefore ask the hospitals to jointly train a model through federated learning. The following scenarios could arise:
\begin{enumerate}
    \item Hospital $A$ is located in a rural area and sees less patients than $B$ and $C$ (which are located in highly populated areas). Therefore, the amount of brain MR images hospital $A$ captures is not enough to train a high performing brain tumor classification model on their own. Hospitals $B$ and $C$ may be asked by the hospital network to enter a federation with hospital $A$ to raise the overall performance ability of hospital $A$ in being able to properly classify brain tumor MR images. In this setting, since hospitals $B$ and $C$ do not necessarily look to benefit from the training (as they could train a high performing model on their own), and rather aim to increase the ability of the least advantaged in the network (hospital $A$), the federation would strive to satisfy \underline{Rawls' difference principle}.
    \item Hospital $A$ receives less funding than hospitals $B$ and $C$ to purchase state-of-the-art MRI equipment and therefore is unable to capture high quality images for training a model. However, the hospital network still wants patients to receive the same quality of care regardless of which hospital they visit. Rather than allocate more funding to hospital $A$, the network may ask hospitals $B$ and $C$ to rectify this inequality by producing a model that performs similarly across all of the hospitals' datasets. In this case, the hospitals would aim to achieve \underline{egalitarian} fairness \cite{donahue2023fairness}.
    \item Assuming that all hospitals have access to the same funding and equipment, it could be the case that hospital $A$ chose to spend less time/money on capturing high quality MR images. In this case, the hospital network may want to reward hospitals $B$ and $C$ for their efforts by ensuring the final model has favorable performance on their local data distributions over hospital $A$'s. In this case, the federation would aim to achieve \underline{desert} fairness\footnote{\cite{donahue2023fairness} refers to this type of fairness as proportional fairness where instead we refer to it as desert fairness.} \cite{donahue2023fairness}.
    \item Suppose that all of the hospitals have access to the same funding and equipment, and choose to spend time and money to capture high quality MR images, but reside in different geographical areas. It could be the case that only certain brain tumor types appear in each geographical region, and therefore, a single hospital would not be able to accurately classify all the different varieties of brain tumors. In this case, the hospital network may have the hospitals enter a federation and jointly train a model that achieves high overall accuracy on all brain tumor types. Here, the hospitals would be striving to achieve \underline{utilitarian} fairness.
\end{enumerate}}

We note that it is not our intention or goal in this work to claim/prove one theory of distributive justice is better than the others. Rather, we believe that all of the distributive justice theories can be valid ethical ideals depending on the setting. This is a main contributor to why we structured UDJ-FL to achieve all four of the theories so that practitioners can choose which best suits their needs.

\subsection{Uncertainty Quantification}
\label{sec:uncertainty_quant}
While uncertainty has long been studied in the field of classical statistics, the machine learning community has only recently begun to estimate and analyze uncertainty in order to produce more reliable model predictions \cite{wimmer2023quantifying}. Often, attention is paid to \textit{predictive uncertainty} which is the uncertainty related to the prediction $\hat{y}\in\mathcal{Y}$ for a specific query $\bm{x}\in\mathcal{X}$. Since the prediction $\hat{y}\in\mathcal{Y}$ is the final outcome of a multitude of different learning and approximation steps, all of the error and uncertainty related to these steps may also contribute to the uncertainty about $\hat{y}$ (i.e., the predictive uncertainty). First, since the mapping of $\mathcal{X}$ to $\mathcal{Y}$ is non-deterministic, the true model output is a conditional probability distribution over $\mathcal{Y}$ rather than a single label $\hat{y}$: $\mathbb{P}(y \mid \bm{x}) = \frac{\mathbb{P}(\bm{x},y)}{\mathbb{P}(\bm{x})}$. This type of uncertainty is termed \textit{aleatoric uncertainty} (or data uncertainty), and since it is not a property of the model, but rather is an inherent property of the underlying data distribution, it is irreducible \cite{abdar2021review}. Potential sources of aleatoric uncertainty include measurement errors, randomness in physical quantities, and the simple fact that $\bm{x}$ may not suffice to explain $y$ \cite{wimmer2023quantifying}\footnote{Throughout this work we use ``data quality'' as a proxy for aleatoric uncertainty without explicitly stating so. Therefore, we consider aleatoric uncertainty in general and assume it can be caused by generalization/measurement/or any other type of error. We believe that regardless of how the error is generated, data quality will always be affected and therefore the motivation for our work holds. In the future, a paper dedicated to the effect of aleatoric uncertainty source on our approach would be interesting to undertake.}. For example, consider flipping a coin. No matter how many times the coin is flipped (to say, use the results to train a prediction model), it is not possible to say with 100\% certainty that the next coin flip will be heads (or tails). In contrast \textit{epistemic uncertainty} (otherwise known as model uncertainty or knowledge uncertainty) occurs due to the model having inadequate knowledge, and therefore this type of uncertainty can be reduced on the basis of additional information \cite{hullermeier2021aleatoric}. For example, the standard monolingual native English speaker is not able to instinctively know the pronunciation and definition of Chinese characters, but when given access to additional resources like a translator, they will most likely be able to derive their meaning. 

In machine learning, predictive uncertainty is often decomposed into aleatoric and epistemic uncertainty using (discrete) \textit{Shannon entropy} \cite{wimmer2023quantifying}:
\begin{equation}
    \mathbb{H}(Y) = -\sum_{y\in\mathcal{Y}}\mathbb{P}(y)\cdot\log\mathbb{P}(y)
\end{equation}
since it has been shown that Shannon entropy additively decomposes into \textit{conditional entropy} (aleatoric uncertainty) and \textit{mutual information} (epistemic uncertainty) \cite{ash2012information}:
\begin{equation}
    \mathbb{H}(Y) = \underbrace{\mathbb{H}(Y\mid\Theta)}_{\substack{\text{aleatoric}}} + \underbrace{I(Y,\Theta)}_{\substack{\text{epistemic}}}
\end{equation}
where $\Theta$ is a random variable of the possible (first-order) model distributions $\bm{\theta}\sim Q$, and $Q$ is a second-order probability distribution over $\bm{\theta}$. Here, $\mathbb{H}(Y\mid\Theta)$ and $I(Y,\Theta)$ can be calculated as:
\begin{equation}
\label{eq:conditional-entropy}
    \begin{aligned}
        & \mathbb{H}(Y\mid\Theta) = \mathbb{E}_{Q}[\mathbb{H}(Y\mid\bm{\theta})]  = -\sum_{k=1}^K\sum_{y\in\mathcal{Y}}-\mathbb{P}(y\mid\theta_k)\log\mathbb{P}(y\mid\theta_k) \\
        &  I(Y,\Theta) = \mathbb{H}(Y) - \mathbb{H}(Y\mid\Theta) = \mathbb{E}_Q[D_{\text{KL}}(\mathbb{P}(Y\mid\Theta) \parallel \mathbb{P}(Y)]
    \end{aligned}
\end{equation}
where the expectation of the conditional entropy is estimated using an ensemble of $K$ models and $D_{\text{KL}}$ denotes Kullback–Leibler divergence. Conditional entropy $\mathbb{H}(Y\mid\Theta)$ gives the uncertainty about $Y$ that remains even if we knew $\Theta$ (e.g., uncertainty in the \textit{data}) and therefore it is often taken to represent aleatoric uncertainty. On the other hand, mutual information is often taken to represent epistemic uncertainty as it is a symmetric measure for the expected information gained about one variable through observing another \cite{wimmer2023quantifying}. Specifically, it quantifies the possible reduction in uncertainty about $Y$ when $\Theta$ is observed. In this work, we are primarily interested in the aleatoric uncertainty of each client's local data distribution since it can be used to define the least advantaged client, or the contribution of a client throughout training. Since we are using deterministic models (feed-forward neural networks), which produce a softmax distribution $\mathbb{P}(y \mid \bm{x}, \theta)$, we can use either the softmax confidence ($\max_c\mathbb{P}(y=c\mid\bm{x}, \theta)$) or the softmax entropy ($H(Y\mid X, \theta)$) as an estimate of the aleatoric uncertainty. In this work, we specifically use the softmax entropy as our aleatoric uncertainty estimate. 

There are several reasons behind our choice of using aleatoric uncertainty to define contribution/advantage over epistemic/predictive uncertainty or the na{\"i}ve choice of using client dataset size. First, we choose to use aleatoric uncertainty over epistemic/predictive uncertainty since it is an \textit{irreducible} quantity while epistemic and predictive uncertainty can be improved with training. When choosing how to define ``least advantaged client'' or ``contribution'' we want to use an easy to calculate, steady, value that remains the same throughout training. Second, if the clients are weighed according to dataset size, we can end up with a scenario where the largest client has high aleatoric uncertainty. In this case, the performance of the global model will be sub-optimal as the largest update is based on a client that has a larger fundamental limiting lower bound (see Appendix \ref{sec:irreducible-bayes} for more discussion). In this work, however, we use aleatoric uncertainty as our weights, which allows us to selectively give clients with less advantage (e.g., higher aleatoric uncertainty), or more contribution (e.g., lower aleatoric uncertainty), proper attention depending on the wanted fairness. 

\section{Methodology}
\label{sec:method}
In this section, we present the formulation for our UDJ-FL framework and show how the choice of hyperparameters determines what distributive justice theory is satisfied by our framework. We begin by defining our federated learning setting. In this work, we assume a horizontal cross-silo federated learning scenario. Therefore, we assume the federation to be comprised of a small number of clients $N$ (where client $i$ has $n_i$ data points) that participate every round and that the data is partitioned horizontally along the examples (e.g., all clients have the same features and labels, however, each client can have a different distributions of them). When considering distributive justice, the concept of ``advantage'' and/or ``contribution'' has to be defined in a way that makes sense to the overall setting. Most federated learning works take ``advantage'' to be those clients who have more data points or who have lower loss throughout training where ``contribution'' is often measured in terms of how many rounds the clients participate in or how many data points are in their local datasets. However, in the cross-silo setting, each client participates the same number of rounds, and as previously discussed, conflating number of data points with advantage or contribution (e.g., as FedAvg does) can cause the model to perform sub-optimally when the largest clients have low quality data. Therefore, in this work, we choose to define advantage and contribution based on the aleatoric uncertainty level of the client as it provides a principled measure for how useful a clients data is. 

\subsection{UDJ-FL}
In \cite{lan2010axiomatic}, the authors construct a family of weighted fairness measures $f_{\beta}(\bm{x}, \bm{q})$ that covers several fairness definitions in resource allocation, including $\alpha$-fairness \cite{mo2000fair} and Jain's index \cite{jain1984quantitative}. Their fairness measure is given as (\cite{lan2010axiomatic} Eq. 69):
\begin{equation}
\label{eq:main-eq}
f_{\beta}(\bm{x}, \bm{q}) = \textrm{sign}(-r(1+r\beta))\left[\sum_{i=1}^Nq_i\left(x_i\right)^{-r\beta}\right]^\frac{1}{\beta}
\end{equation}
where $\bm{x} = [x_i]_{i=1}^N$ is a resource allocation vector with $x_i$ being the resource allocated to user $i$, $\bm{q} = [q_i]_{i=1}^N$ is a vector of weights where $q_i$ gives the relative importance of client $i$ in quantifying the fairness of the system, $r,\beta\in\mathbb{R}$ are constants where $r$ controls the growth rate of maximum fairness as population size $N$ increases and the choice of $\beta$ recovers the different fairness definitions. For instance, the authors show that $\alpha$-fairness can be derived from Eq. \ref{eq:main-eq} by letting $r=1-\frac{1}{\beta}$. $\alpha$-fairness was originally proposed as a fairness measure for resource allocation where the degree of fairness is set by the parameter $\alpha$ which controls the trade-off between utility and fairness. Specifically, Eq. \ref{eq:main-eq} becomes:
\begin{equation}
\label{eq:derived-alpha}
    |f_\beta(\bm{x},\bm{q})|^\beta = (1-\beta)\sum_{i=1}^N\frac{q_i}{1-\beta}x_i^{1-\beta}
\end{equation}
where $\alpha$-fairness is formulated as:
\begin{equation}
\label{eq:alpha}
    U_{\alpha=\beta} = \begin{cases}
        \frac{1}{1-\beta}x^{1-\beta} & \beta\geq0,\beta\neq1\\
        \log(x) & \beta=1
    \end{cases}
\end{equation}
We derive our formulation for UDJ-FL using ideas similar to how Eq. \ref{eq:derived-alpha} was constructed. Namely, our federated learning objective for UDJ-FL can be written as:
\begin{equation}
\label{eq:objective-function}
\begin{aligned}
    \min_{\theta}  h(\theta)  = |f_{\beta}(\bm{H}(\theta), \bm{\upsilon})|^\beta =\sum_{i=1}^N\upsilon_i^\gamma H_i(\theta)^{r\beta}
\end{aligned}
\end{equation}
where $\bm{H}(\theta) = \left[H_i(\theta)\right]_{i=1}^N$, $H_i(\theta)$ is the empirical risk of client $i$, $\bm{\upsilon}=\left[\frac{\upsilon_i}{\sum_{j=1}^N\upsilon_j}\right]_{i=1}^N$ are the aleatoric uncertainty-based weights of the clients as defined below, and $\gamma\in\{-1, 0, 1\}$. We also note that the change in sign of the exponent $-r\beta$ to $r\beta$ is due to federated learning minimizing costs rather than distributing utilities as in resource allocation. Therefore, $\max\min$ in resource allocation corresponds to $\min\max$ in federated learning \cite{li2019fair}. Since aleatoric uncertainty is taken per example $\bm{x}$ (see Section \ref{sec:uncertainty_quant}), we determine each client's local dataset quality (i.e., their overall aleatoric uncertainty) by taking an average over all calculated aleatoric uncertainty quantities (i.e., we calculate $\mathbb{E}[\mathbb{H}(Y\mid X, \theta)]$) and denote this quantity as $\upsilon$:
\begin{equation}
\label{eq:upsilon}
\begin{aligned}
    \upsilon & = \frac{1}{|D_{tr}|}\sum_{\bm{x}\in D_{tr}}\mathbb{H}(Y\mid \bm{x}, \theta) = \frac{1}{|D_{tr}|}\sum_{\bm{x}\in D_{tr}}\sum_{c\in \mathcal{C}}\frac{e^{f_c(\bm{x}, \theta)}}{\sum_{c'\in \mathcal{C}}e^{f_{c'}(\bm{x}, \theta)}}\log\left(\frac{e^{f_c(\bm{x}, \theta)}}{\sum_{c'\in \mathcal{C}}e^{f_{c'}(\bm{x}, \theta)}}\right)
\end{aligned}
\end{equation}
where $|D_{tr}|$ is the size of the client's train set, $\mathcal{C} = [1, \dots, C]$ are the possible classes, and $f_c(\bm{x}, \theta)$ is the logit produced by the model for the $c$-th class on data point $\bm{x}$. By using the overall aleatoric uncertainty as our weights, as compared to weighing clients according to dataset size, we are able to properly define client contribution and advantage in federated learning to achieve the four distributive justice types.

\subsection{Rawls' Difference Principle Fairness}
In Rawls' difference principle, the \textit{only} allowable difference in material distribution are those that help the least advantaged in society. We note, however, in federated learning it is difficult to ensure \textit{only} the least advantaged improves. Therefore, we slightly relax Rawls' difference principle to the following:
\begin{definition}[Rawls' Difference Principle Fairness for Federated Learning]
For trained models $\theta$ and $\bar{\theta}$, we say that model $\theta$ provides a more Rawls' DP fair solution to the federated learning objective than model $\bar{\theta}$ if the difference between the absolute increase in performance (relative to training under FedAvg) achieved by the client with the highest (worst) local dataset aleatoric uncertainty score $\upsilon_i$ and the the average absolute increase of all other clients (termed absolute increase difference) under model $\theta$ is higher than the difference in performances achieved under model $\bar{\theta}$. 
\end{definition}

In \cite{barsotti2024minmax} the authors show how Rawls' difference principle can be encapsulated by $\min\max$ fairness. Additionally, it has been shown that as $\alpha\to\infty$, $\alpha$-fairness converges to $\min\max$ fairness \cite{mo2000fair}. Similar to \cite{lan2010axiomatic}, we can construct $\alpha$-fairness from Eq. \ref{eq:objective-function} by setting $r=1+\frac{1}{\beta}$ and $\gamma=1$ which produces our UDJ-FL objective function for Rawls' difference principle:
\begin{equation}
\label{eq:rawls-opt}
\begin{aligned}
    \min_{\theta}  h(\theta) = (1+\beta)\sum_{i=1}^N\frac{\upsilon_i}{(1+\beta)}H_i(\theta)^{1+\beta}
\end{aligned}
\end{equation}
By letting $\beta\to\infty$, we recover $\min\max$ fairness: $\min_\theta h(\theta) = \max_i|\upsilon_iH_i(\theta)|$ and therefore Rawls' difference principle. Here, we set $\gamma=1$ as it places more importance on the clients with high aleatoric uncertainty values -- which we define as the least advantaged clients. To measure Rawlsian fairness, we use the following:
\begin{equation}
\label{eq:Psi}
    \Psi = \psi_{i=\arg\max{\bm{\upsilon}}} - \frac{1}{N-1}\sum_{i\neq\arg\max{\bm{\upsilon}}}\psi_i,  \;\;\;\psi_i = \left(Acc^{Method}_i - Acc^{FedAvg}_i\right)
\end{equation}
where $\psi_i$ is the absolute difference between client $i$'s accuracy under the chosen fairness method and the accuracy the client obtained under standard federated averaging. Eq. \ref{eq:Psi} will be positive when the absolute increase of the client with the highest aleatoric uncertainty value is greater than the average absolute increase of all other clients in the federation. A higher value of $\Psi$ indicates greater Rawls' DP fairness.

\begin{wraptable}{r}{0.55\textwidth}
\small
\centering
\vspace{-1cm}
\caption{\label{tab:functions}Derived distributive justice optimization functions and the hyperparameters used in the derivations.}
\renewcommand{\arraystretch}{1.1}
\begin{tabular}{|c|c|ccc|}
\hline
\multirow{2}{*}{Method} & \multirow{2}{*}{Equation} & \multicolumn{3}{c|}{Hyperparameters} \\ \cline{3-5} 
 &  & \multicolumn{1}{c|}{$r$} & \multicolumn{1}{c|}{$\beta$} & \multicolumn{1}{c|}{$\gamma$} \\ \hline
Rawls' DP & $(1+\beta)\sum_{i=1}^N\frac{\upsilon_i}{(1+\beta)}H_i(\theta)^{1+\beta}$ & $1+\frac{1}{\beta}$ & $>0$ & 1 \\
Egalitarian & $\sum_{i=1}^N\upsilon_iH_i(\theta)$ & $1$  & $1$ & 1\\
Utilitarian & $\sum_{i=1}^N\frac{1}{\upsilon_i}H_i(\theta)$ & $1+\frac{1}{\beta}$ & $\to 0$ & -1 \\
Desert & $\frac{1}{N}\sum_{i=1}^N H_i(\theta)^{-\beta_i}$ & 1 & Eq. \ref{eq:individual-beta} & 0 \\ \hline
\end{tabular}
\end{wraptable}
\subsection{Utilitarian Fairness}
Utilitarian fairness in federated learning can simply be seen as trying to maximize the global model's performance. We define the utilitarian fairness of a federated learning model as follows:
\begin{definition}[Utilitarian Fairness for Federated Learning]
For trained models $\theta$ and $\bar{\theta}$, we say that model $\theta$ provides a more Utilitarian fair solution to the federated learning objective than model $\bar{\theta}$ if the average performance obtained by model $\theta$ on the $N$ clients is greater than the average performance obtained by model $\bar{\theta}$ on the $N$ clients.
\end{definition}

Utilitarian fairness can be recovered from $\alpha$-fairness when $\alpha$=0 \cite{lan2010axiomatic}. Therefore, using Eq. \ref{eq:rawls-opt} as a starting point, we can obtain an objective function for utilitarian fairness by letting $r=1+\frac{1}{\beta}$, $\beta\to0$, and $\gamma=-1$:
\begin{equation}
\label{eq:util-opt}
\begin{aligned}
    \min_{\theta}  h(\theta) = \sum_{i=1}^N\frac{1}{\upsilon_i}H_i(\theta)
\end{aligned}
\end{equation}
Here, by setting $\gamma=-1$ we put more emphasis on clients with low aleatoric uncertainty scores which allows the global model to learn from higher quality data samples. We consider overall global model accuracy as the metric for measuring utilitarian fairness.

\subsection{Desert Fairness}
In desert-based fairness, we are concerned with distributing material goods and services according to a person's contribution to the designated task. In federated learning \textit{contribution} can be defined in multiple ways. For instance, it could be taken literally and be defined as the number of times the client participates, or as how many data examples the user contributes (such as is done in FedAvg). However, since we are working in the cross-silo federated setting, equating ``number of rounds participated'' to ``contribution'' is illogical since all clients participate in every round. And, as previously mentioned, weighing clients by the size of their local datasets can lead to poor overall performance if the client with the largest dataset has poor quality data. Therefore, we propose to measure \textit{contribution} as \textit{data quality}. In this work, we define desert fairness for federated learning as:

\begin{definition}[Desert Fairness for Federated Learning]
For trained models $\theta$ and $\bar{\theta}$, we say that model $\theta$ provides a more Desert fair solution to the federated learning objective than model $\bar{\theta}$ if the performance of model $\theta$ on the $N$ clients is distributed more proportionate to the inverse of the clients' aleatoric uncertainty scores (e.g., clients with lower aleatoric uncertainty obtain higher performance from the model) than the distribution obtained under model $\bar{\theta}$ on the $N$ clients.
\end{definition}

The objective function to achieve desert-based fairness is quite similar to the objective function for Rawls' DP fairness. However, instead of considering the update of the most disadvantaged user to be more important, we instead want to place more importance on the clients with the higher quality data. Therefore, we let $\beta\leq0$ and $r=1$. Additionally, we set $\gamma=0$ since extra client weighing, in addition to scaling the clients' losses, will cause the distribution of final model accuracies to be less similar to the distribution of client local aleatoric uncertainty values. More importantly, to ensure the proper distribution of client accuracies, we set $\beta$ individual for each client. Specifically, we set $\beta_i$ as:
\begin{equation}
\label{eq:individual-beta}
    \bm{\beta} = \left[\beta_i=\frac{\frac{1}{\upsilon_i}}{\sum_{j=1}^N\frac{1}{\upsilon_j}}\right]_{i=1}^N
\end{equation}

Starting from Eq. \ref{eq:objective-function}, our objective function for desert fairness becomes: 
\begin{equation}
\label{eq:desert-opt}
\begin{aligned}
    \min_{\theta}  h(\theta) =\frac{1}{N}\sum_{i=1}^N H_i(\theta)^{-\beta_i}
\end{aligned}
\end{equation}

To measure desert fairness, rather than simply comparing the accuracy of the client with the lowest aleatoric uncertainty, we want to measure how close the distribution of client accuracies is to the distribution of the clients local dataset aleatoric uncertainty values. Therefore, we use the Pearson correlation coefficient between the clients' aleatoric uncertainty values and the clients' final local dataset accuracies: 
\begin{equation}
    \label{eq:pearson}
    r_{\bm{\upsilon}, \bm{A}} = \frac{\sum_{i=1}^N(\upsilon_i - \bar{\bm{\upsilon}})(A_i-\bar{\bm{A}})}{\sqrt{\sum_{i=1}^N(\upsilon_i - \bar{\bm{\upsilon}})^2}\sqrt{\sum_{i=1}^N(A_i-\bar{\bm{A}})^2}}
\end{equation}
where $A_i$ is the accuracy of client $i$, $\bar{\bm{\upsilon}}=\frac{1}{N}\sum_{i=1}^N\upsilon_i$, and $\bar{\bm{A}} = \frac{1}{N}\sum_{i=1}^NA_i$. Here, $r_{\bm{\upsilon}, \bm{A}}\to-1$ implies greater desert fairness.

\subsection{Egalitarian Fairness}
In egalitarian fairness, the goal is to distribute resources as equally as possible among a population. In the case of federated learning, egalitarian fairness can be seen as each user obtaining similar accuracies on their local data distributions after the federated learning process has completed. We can define egalitarian fairness over a federated learning performance distribution as:

\begin{definition}[Egalitarian Fairness for Federated Learning]
For trained models $\theta$ and $\bar{\theta}$, we say that model $\theta$ provides a more Egalitarian fair solution to the federated learning objective than model $\bar{\theta}$ if the performance obtained by model $\theta$ on the $N$ clients is distributed more uniformly (e.g., in terms of standard deviation) than the performance achieved under model $\bar{\theta}$ on the $N$ clients. 
\end{definition}

Often in federated learning, Rawls' DP fairness can be conflated with egalitarian fairness \cite{li2019fair}. This is not unexpected, since raising the overall contribution of the least advantaged client (to raise their overall performance) can limit the potential of the other clients, thereby making the final performances of the clients' more similar. In this work however, we make an explicit difference between egalitarian and Rawls' DP fairness. More specifically, when analyzing Rawls' DP fairness
\begin{wrapfigure}{l}{0.35\textwidth}
    \centering
    \includegraphics[width=0.35\textwidth]{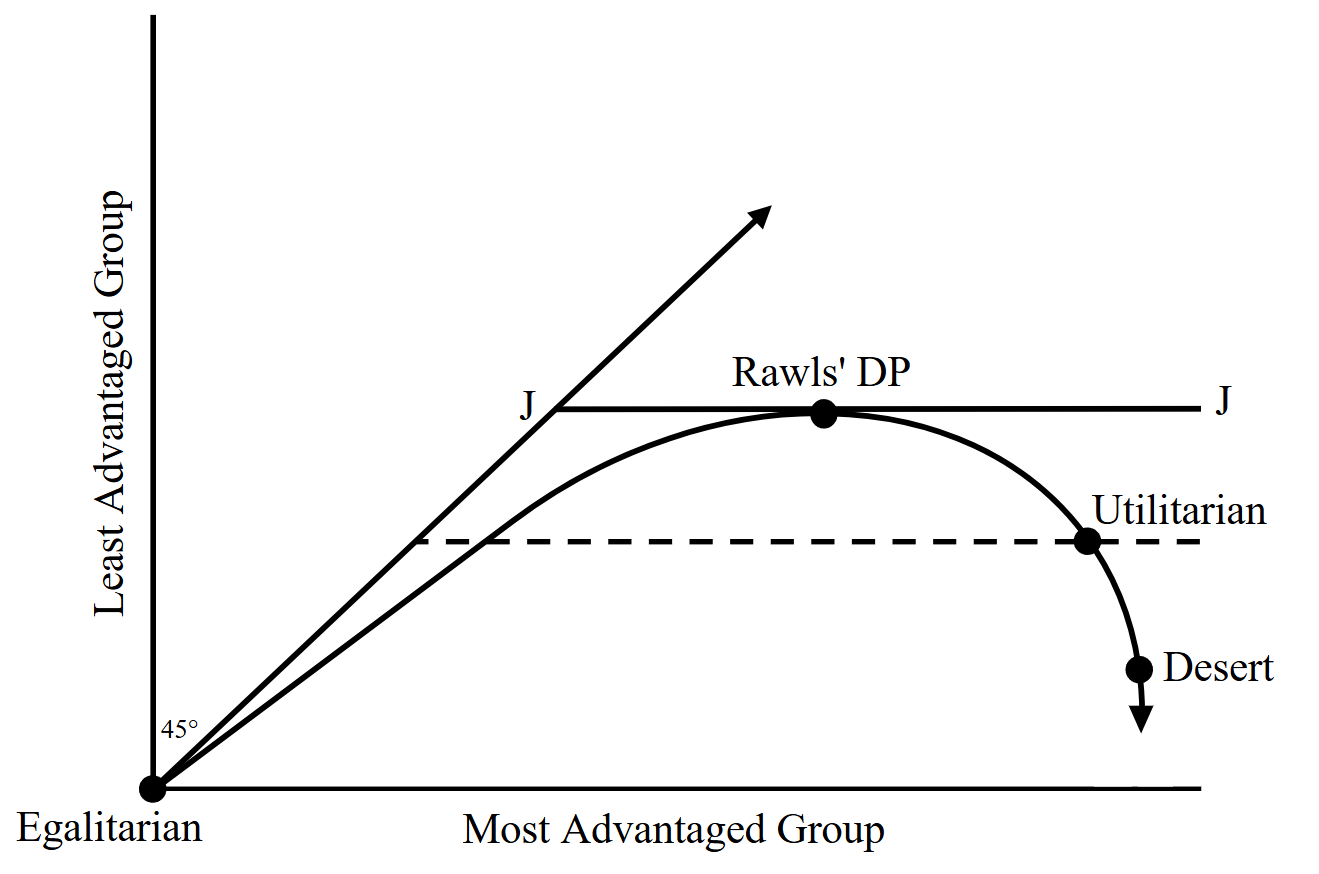}
    \caption{Reproduced from \cite{rawls1971atheory} and \cite{lan2010axiomatic}. Relationship of the four distributive justice theories in terms of the benefit to the most advantaged group and the least advantaged group. The fairness function presented in Eq. \ref{eq:objective-function} can generate any point on the curve depending on the set hyperparameters.}
    \label{fig:dj-graph}
\end{wrapfigure} we only consider the absolute difference between the results under federated averaging and results under the chosen fairness method (Eq. \ref{eq:Psi}) and when analyzing egalitarian fairness methods we only consider the overall spread (in terms of standard deviation) of the clients' final local accuracies. While power scaling is helpful in terms of placing the most importance on the least advantaged client, it does not necessarily make the standard deviation of the clients' accuracies smaller -- which we show in our experiments in Section \ref{sec:experiment}. To construct our egalitarian objective, we allow $r=1$, $\beta=1$, and $\gamma=1$. Our objective function for egalitarian fairness can be written as:
\begin{equation}
\label{eq:egal-opt}
    \quad\quad\quad\quad\quad\quad\quad\quad\quad\quad\quad\quad\quad\quad\quad\quad\quad\quad\min_{\theta}  h(\theta) = \sum_{i=1}^N\upsilon_iH_i(\theta)
\end{equation}
Fig. \ref{fig:dj-graph} shows the relationship between all four distributive justice theories along the efficiency frontier between the least advantaged group and most advantaged group. Here, egalitarian appears at the origin as we strive to have a distribution that equally benefits the least advantaged group (highest aleatoric uncertainty) and the most advantaged group (lowest aleatoric uncertainty). This is one of the reasoning behind setting $r=1$ and $\beta=1$ to ensure egalitarian fairness. However, we still want to increase the contribution of clients with high aleatoric uncertainty so that they can receive equal performance to those clients with high aleatoric uncertainty, therefore we set $\gamma=1$. 

\subsection{Solving UDJ-FL}
We borrow a similar formulation to \cite{li2019fair} to solve our UDJ-FL objectives efficiently. Namely, the authors of \cite{li2019fair} discuss that it is common in fairness applications to have to train a family of objective functions under different hyperparameter settings to find the desired fairness/utility trade-off. In our case, a practitioner may need to train the UDJ-FL objective under multiple $\beta$ values to find the $\beta$ that achieves the desired fairness constraint. However, \cite{li2019fair} note that solving such a family of objective functions requires step-size tuning for every scenario, which can become costly. They propose to overcome this issue by estimating the local Lipschitz constant for the family of objectives by tuning the step size on just one fairness value (e.g., $\beta=0$). Then, the step-size can be manually adjusted using the estimated Lipschitz value in each fairness scenario. The following Lemma (based on Lemma 3 from \cite{li2019fair}) formalizes the relationship between the Lipschitz constant $L$ for $\beta=0$ and $\beta\neq0$:
\begin{lemma}
\label{lem}
    If the non-negative function $h(\cdot)$ has a Lipschitz gradient with constant $L$, then for any $\beta$ and at any point $\theta$,
    \begin{equation}
    \label{eq:lipschitz}
        L_\beta(\theta) = r\beta\upsilon^\gamma\left(Lh(\theta)^{r\beta-1}+(r\beta-1)h(\theta)^{r\beta-2}||\nabla h(\theta)||^2\right)
    \end{equation}
    is an upper-bound for the local Lipschitz constant of the gradient of $\upsilon^\gamma h(\cdot)^{r\beta}$ at point $\theta$.
\end{lemma}
\begin{proof}
    At any point $\theta$, we can compute the Hessian $\nabla^2\left(\upsilon^\gamma h(\theta)^{r\beta}\right)$ as:
    \begin{equation}
        \nabla^2\left(\upsilon^\gamma h(\theta)^{r\beta}\right) = r\beta\upsilon^\gamma h(\theta)^{r\beta-1}\underbrace{\nabla^2h(\theta)}_{\preceq L \times I} + r\beta\upsilon^\gamma(r\beta-1)h(\theta)^{r\beta-2}\underbrace{\nabla h(\theta)\nabla^Th(\theta)}_{\preceq||\nabla h(\theta)||^2\times I}
    \end{equation}
    As a result, $||\nabla^2\upsilon^\gamma h(\theta)^{r\beta}||_2 \leq L_\beta(\theta) =r\beta\upsilon^\gamma\left(Lh(\theta)^{r\beta-1}+(r\beta-1)h(\theta)^{r\beta-2}||\nabla h(\theta)||^2\right)$.
\end{proof}

\begin{wrapfigure}{l}{0.5\textwidth}
\vspace{-.9cm}
    \begin{minipage}{0.5\textwidth}
      \begin{algorithm}[H]
    \footnotesize
    \begin{algorithmic}[1]
     \caption{\label{alg:UDJ-FL}UDJ-FL}
        \renewcommand{\algorithmicrequire}{\textbf{Input:}}
        \REQUIRE $S$, $T$, $E$, $\beta$, $\gamma$, $r$, $L = 1/\eta$, $i=1,\dots,N$
            \STATE Each client $i$ trains a model individually for $S$ rounds, and on round $S$ generate $\upsilon_i$ according to Eq. \ref{eq:upsilon} and send it to the server
            \STATE Server initializes $\theta^0$
            \FOR{each round $t=1, 2, \dots, T$}
            \STATE Server sends $\theta^t$ to all clients
            \STATE Each client $i$ updates $\theta^t$ for $E$ epochs of SGD on $H_i$ with step-size $\eta$ to obtain $\bar{\theta}^{t+1}$
            \STATE Each client $i$ computes: 
            \begin{equation*}
                \begin{aligned}
                    & \Delta\theta_i^t = L\left(\theta^t - \bar{\theta}_i^{t+1}\right)\\
                    & \Delta_i^t = r\beta\upsilon_i^\gamma H_i(\theta^t)^{r\beta-1}\Delta\theta_i^t \\
                    & g_i^t = r\beta\upsilon_i^\gamma\left(LH_i(\theta^t)^{r\beta-1}+(r\beta-1)H_i(\theta^t)^{r\beta-2}||\Delta\theta_i^t||^2\right)
                \end{aligned}
            \end{equation*}
            \STATE Each client $i$ sends $\Delta_i^t$ and $g_i^t$ back to the server
            \STATE Server updates $\theta^{t+1}$ as:
            \begin{equation*}
                \theta^{t+1} = \theta^t - \frac{\sum_{i=1}^N\Delta_i^t}{\sum_{i=1}^Ng_i^t}
            \end{equation*}
            \ENDFOR
    \end{algorithmic}
\end{algorithm}
    \end{minipage}
  \end{wrapfigure}
We provide the pseudo-code for UDJ-FL in Alg. \ref{alg:UDJ-FL} which is based on the $q$-FedAvg approach presented in \cite{li2019fair}. The UDJ-FL learning process begins by each client individually training a model for $S$ rounds in order to approximate their local dataset's aleatoric uncertainty using Eq. \ref{eq:upsilon}. We choose for the clients to calculate their aleatoric uncertainty value before federated training starts, rather than calculate it each federated round, both to save computational/communication resources, as well as the pre-calculated value being a better estimate of the clients true aleatoric uncertainty level. After each client calculates their aleatoric uncertainty score $\upsilon_i$, they send it to the server and federated training begins. The federated training process is similar to the $q$-FedAvg process described in \cite{li2019fair}. Namely, when $\beta \neq 0$, the $H_i(\theta)^{r\beta}$ term is no longer an empirical average of the loss over all local samples due to the $r\beta$ exponent. Therefore, the averaging scheme is altered so that the step sizes are inferred from the upper bound of the local Lipschitz constant (Eq. \ref{eq:lipschitz}) and the gradient is replaced with the local updates that are obtained by running SGD locally (first equation in line 6). We defer to \cite{li2019fair} for a more in-depth discussion of the changes made to adapt for the $r\beta$ exponent. We also note that while UDJ-FL technically has three hyperparameters ($\beta$, $r$, and $\gamma$) \textit{only $\beta$ has to be tuned}. Both $r$ and $\gamma$ are set by the practitioner before training according to the desired type of fairness. We list the proper choices for $\gamma$ and $r$ in Table \ref{tab:functions} along with the best $\beta$ ranges to test for the specific fairness types.

\begin{table}[t!]
\setlength{\tabcolsep}{5pt}
\renewcommand{\arraystretch}{1.1}
    \centering
    \caption{\label{tab:baseline-distjust}Dataset details and solo test accuracies. Local data = client tests on only their test set. Full Data = client tests on a dataset containing all clients test data to analyze generalizability of the local models.}
    \begin{tabular}{cc|ccccc|}
    \cline{3-7}
     &  & \multicolumn{5}{c|}{\begin{tabular}[c]{@{}c@{}}Client\\ Num. Clean/Ambiguous Shards\end{tabular}} \\ \hline
    \multicolumn{1}{|c|}{Dataset} & Metric & \begin{tabular}[c]{@{}c@{}}1\\ 19/1\end{tabular} & \begin{tabular}[c]{@{}c@{}}2\\ 15/5\end{tabular} & \begin{tabular}[c]{@{}c@{}}3\\ 10/10\end{tabular} & \begin{tabular}[c]{@{}c@{}}4\\ 5/10\end{tabular} & \begin{tabular}[c]{@{}c@{}}5\\ 1/19\end{tabular} \\ \hline
    \multicolumn{1}{|c|}{\multirow{3}{*}{\begin{tabular}[c]{@{}c@{}}Dirty-\\ MNIST\end{tabular}}} & Local Data & 96.39\textsubscript{{\footnotesize$\pm$1.14}} & 94.50\textsubscript{{\footnotesize$\pm$0.50}} & 93.75\textsubscript{{\footnotesize$\pm$0.88}} & 91.25\textsubscript{{\footnotesize$\pm$1.46}} & 91.00\textsubscript{{\footnotesize$\pm$0.88}} \\
    \multicolumn{1}{|c|}{} & All Data & 76.08\textsubscript{{\footnotesize$\pm$14.06}} & 77.25\textsubscript{{\footnotesize$\pm$5.74}} & 69.78\textsubscript{{\footnotesize$\pm$3.78}} & 68.23\textsubscript{{\footnotesize$\pm$2.92}} & 65.58\textsubscript{{\footnotesize$\pm$6.38}} \\
    \multicolumn{1}{|c|}{} & $\upsilon_i$ & 0.03\textsubscript{{\footnotesize$\pm$0.01}} & 0.06\textsubscript{{\footnotesize$\pm$0.01}} & 0.12\textsubscript{{\footnotesize$\pm$0.01}} & 0.19\textsubscript{{\footnotesize$\pm$0.04}} & 0.23\textsubscript{{\footnotesize$\pm$0.04}} \\ \hline
    \multicolumn{1}{|c|}{\multirow{3}{*}{\begin{tabular}[c]{@{}c@{}}CURE\\ -TSR\end{tabular}}} & Local Data & 98.05\textsubscript{{\footnotesize$\pm$0.48}} & 96.94\textsubscript{{\footnotesize$\pm$1.58}} & 96.25\textsubscript{{\footnotesize$\pm$0.72}} & 93.19\textsubscript{{\footnotesize$\pm$1.97}} & 89.44\textsubscript{{\footnotesize$\pm$3.96}} \\
    \multicolumn{1}{|c|}{} & All Data& 71.55\textsubscript{{\footnotesize$\pm$12.45}} & 67.21\textsubscript{{\footnotesize$\pm$0.34}} & 56.89\textsubscript{{\footnotesize$\pm$6.02}} & 48.11\textsubscript{{\footnotesize$\pm$15.81}} & 46.39\textsubscript{{\footnotesize$\pm$6.34}} \\
    \multicolumn{1}{|c|}{} & $\upsilon_i$ & 0.36\textsubscript{{\footnotesize$\pm$0.02}} & 0.49\textsubscript{{\footnotesize$\pm$0.06}} & 0.50\textsubscript{{\footnotesize$\pm$0.09}} & 0.82\textsubscript{{\footnotesize$\pm$0.06}} & 0.91\textsubscript{{\footnotesize$\pm$0.04}} \\ \hline
    \end{tabular}
\end{table}

\section{Experimental Evaluation}
\label{sec:experiment}
\subsection{Dataset, Models, and Baselines}
\label{sec:data}
In this work, we consider a federation of 5 clients, each of which have datasets that are locally IID, while globally non-IID with the other four clients, and use a single hidden layer MLP as our model. We utilize two datasets in our experiments: Dirty-MNIST and CURE-TSR. Dirty-MNIST \cite{mukhoti2023deep} is a modified version of the popular MNIST dataset \cite{lecun1998gradient} that contains additional ambiguous digits (Ambiguous-MNIST). The CURE-TSR \cite{Temel2017_NIPS} dataset is made up of 2 million traffic sign instances over 14 different classes as well as 13 different challenge conditions (e.g., lens glare, rain, snow). In this work, we choose to utilize two of the 13 challenge conditions, 1) no challenge, which acts as the base clean images, and 2) lens glare to introduce ambiguity into the dataset. For both datasets, we control the aleatoric uncertainty level for each client by giving them more/less ambiguous examples in their local dataset. We additionally follow the original non-IID partitioning strategy from \cite{mcmahan2017communication} to split the data among the clients. We sort the datasets' clean and ambiguous images separately according to the image class and create shards that are distributed to the clients. We assign 20 shards to each client where each client receives a different mixture of clean and ambiguous shards to vary the aleatoric uncertainty levels. We list the number of clean/ambiguous shards each client receives in Table \ref{tab:baseline-distjust}. For a baseline, each client trains their own local MLP model for 500 epochs with a learning rate of 0.001 and a batch size of 128. The results of solo training are shown in Table \ref{tab:baseline-distjust}. Here, \textit{local data} refers to the test accuracy on the clients' own data distribution while \textit{all data} refers to the test accuracy on a held out set of clean images that all clients test to show the generalizability of the models. We additionally note that in testing the accuracy of the created global model, we use a held out test set that only contains clean images. We do this with the assumption that while some clients may have unreliable training data, the average test query is of a clean image. We compare UDJ-FL against multiple baselines that span all four of the distributive justice types and provide more detail in Appendix \ref{app:experimental}. In all cases but desert, we test an alternate version of UDJ-FL where instead of weighing clients according to their aleatoric uncertainty level, we weigh clients equally $(\gamma=0)$ to show the importance of using aleatoric uncertainty as weights. In desert-based fairness, however, we intentionally set $\gamma=0$ to avoid the distribution of client final accuracies diverging from the distribution of the clients' local aleatoric uncertainty values, and therefore, in the desert experiments, we use aleatoric weighing as the alternate. 

\begin{table}[t!]
\centering
\caption{Results for DirtyMNIST and CURE-TSR trained using various fair federated learning algorithms. Acc: Accuracy in percentage. $\max_{\upsilon_i}$: Client with highest aleatoric uncertainty score. $\min_{\upsilon_i}$: Client with lowest aleatoric uncertainty score. ($\uparrow$): Higher values favorable. ($\downarrow$): Lower values favorable. \textbf{Bold:} Best result. \underline{Underline:} Second best result. {\color{gray} Gray Results}: unimportant to the fairness type, but are listed for completeness.}
\label{tab:results}
\renewcommand{\arraystretch}{1.3}
\resizebox{\textwidth}{!}{%
\begin{tabular}{cc|cccccc|cccccc|}
\cline{3-14}
 &  & \multicolumn{6}{c|}{DirtyMNIST} & \multicolumn{6}{c|}{CURE-TSR} \\ \hline
\multicolumn{1}{|c|}{Fairness} & Method: Parameters & \multicolumn{1}{c|}{Global Acc \tiny{($\uparrow$)}}  & \multicolumn{1}{c|}{Acc$_{\max{\upsilon_i}}$ \tiny{($\uparrow$)}} & \multicolumn{1}{c|}{Acc$_{\min{\upsilon_i}}$ \tiny{($\uparrow$)}} & \multicolumn{1}{c|}{STD \tiny{($\downarrow$)}} & \multicolumn{1}{c|}{$\Psi$ Eq. \ref{eq:Psi} \tiny{($\uparrow$)}} & $r_{\bm{\upsilon},\bm{A}}$ Eq. \ref{eq:pearson} \tiny{($\downarrow$)} & \multicolumn{1}{c|}{Global Acc \tiny{($\uparrow$)}} & \multicolumn{1}{c|}{Acc$_{\max{\upsilon_i}}$ \tiny{($\uparrow$)}} & \multicolumn{1}{c|}{Acc$_{\min{\upsilon_i}}$ \tiny{($\uparrow$)}} & \multicolumn{1}{c|}{STD \tiny{($\downarrow$)}} & \multicolumn{1}{c|}{$\Psi$ Eq. \ref{eq:Psi} \tiny{($\uparrow$)}} & $r_{\bm{\upsilon},\bm{A}}$ Eq. \ref{eq:pearson} \tiny{($\downarrow$)}\\ \hline

\multicolumn{1}{|c|}{-} & FedAvg & 93.83\textsubscript{{\footnotesize$\pm$1.27}} & 87.33\textsubscript{{\footnotesize$\pm$2.47}} & 96.00\textsubscript{{\footnotesize$\pm$0.52}} & 4.03\textsubscript{{\footnotesize$\pm$0.99}} & - & -0.90\textsubscript{{\footnotesize$\pm$0.09}} & 97.73\textsubscript{{\footnotesize$\pm$0.06}} & 93.63\textsubscript{{\footnotesize$\pm$1.46}} & 97.10\textsubscript{{\footnotesize$\pm$2.33}} & 1.84\textsubscript{{\footnotesize$\pm$0.51}} & - & -0.61\textsubscript{{\footnotesize$\pm$0.51}}\\ \hline

\multicolumn{1}{|c|}{\multirow{7}{*}{Utilitarian}} & PropFair: $M=3$ & 91.74\textsubscript{{\footnotesize$\pm$1.59}} & 89.13\textsubscript{{\footnotesize$\pm$1.05}} & 94.10\textsubscript{{\footnotesize$\pm$0.26}} & {\color{gray}2.62\textsubscript{{\footnotesize$\pm$0.53}}} & {\color{gray}2.30\textsubscript{{\footnotesize$\pm$1.53}}} & {\color{gray}-0.79\textsubscript{{\footnotesize$\pm$0.26}}}& 98.00\textsubscript{{\footnotesize$\pm$0.62}} & 90.83\textsubscript{{\footnotesize$\pm$2.95}} & 97.20\textsubscript{{\footnotesize$\pm$1.28}} & {\color{gray}3.07\textsubscript{{\footnotesize$\pm$1.03}}}  & {\color{gray}-2.28\textsubscript{{\footnotesize$\pm$3.33}}} & 
{\color{gray}-0.82\textsubscript{{\footnotesize$\pm$0.14}}}\\

\multicolumn{1}{|c|}{} & \begin{tabular}[c]{@{}c@{}}$q$-FedAvg: $q$=0\end{tabular} & 91.24\textsubscript{{\footnotesize$\pm$1.45}} & 88.60\textsubscript{{\footnotesize$\pm$1.71}} & 93.37\textsubscript{{\footnotesize$\pm$0.40}} & {\color{gray}2.20\textsubscript{{\footnotesize$\pm$0.62}}} & {\color{gray}1.99\textsubscript{{\footnotesize$\pm$1.53}}} & {\color{gray}-0.79\textsubscript{{\footnotesize$\pm$0.22}}} & \underline{98.57\textsubscript{{\footnotesize$\pm$0.60}}} & 98.07\textsubscript{{\footnotesize$\pm$0.87}} & 97.90\textsubscript{{\footnotesize$\pm$0.69}} & {\color{gray}0.56\textsubscript{{\footnotesize$\pm$0.29}}} & {\color{gray}3.18\textsubscript{{\footnotesize$\pm$0.58}}} & 
{\color{gray}-0.56\textsubscript{{\footnotesize$\pm$0.21}}}\\

\multicolumn{1}{|c|}{} & \begin{tabular}[c]{@{}c@{}}$q$-FedAvg: $q$=0.1\end{tabular} & \underline{93.31\textsubscript{{\footnotesize$\pm$1.27}}} & 88.37\textsubscript{{\footnotesize$\pm$1.96}} & 94.57\textsubscript{{\footnotesize$\pm$0.75}} & {\color{gray}3.23\textsubscript{{\footnotesize$\pm$0.76}}} & {\color{gray}1.48\textsubscript{{\footnotesize$\pm$1.97}}} &
{\color{gray}-0.81\textsubscript{{\footnotesize$\pm$0.17}}}& 98.43\textsubscript{{\footnotesize$\pm$0.40}} & 97.10\textsubscript{{\footnotesize$\pm$1.93}} & 98.07\textsubscript{{\footnotesize$\pm$0.87}} & {\color{gray}0.94\textsubscript{{\footnotesize$\pm$0.62}}} & {\color{gray}2.28\textsubscript{{\footnotesize$\pm$1.66}}} & 
{\color{gray}-0.28\textsubscript{{\footnotesize$\pm$0.27}}}\\

\multicolumn{1}{|c|}{} & \begin{tabular}[c]{@{}c@{}}UDJ-FL: $\beta=0, \gamma=0$\end{tabular} & 91.08\textsubscript{{\footnotesize$\pm$1.86}} & 88.23\textsubscript{{\footnotesize$\pm$2.59}} & 93.10\textsubscript{{\footnotesize$\pm$0.26}} & {\color{gray}2.43\textsubscript{{\footnotesize$\pm$1.00}}} & {\color{gray}1.62\textsubscript{{\footnotesize$\pm$2.86}}} & 
{\color{gray}-0.84\textsubscript{{\footnotesize$\pm$0.16}}}& 98.44\textsubscript{{\footnotesize$\pm$0.86}} & 97.37\textsubscript{{\footnotesize$\pm$2.06}} & 97.80\textsubscript{{\footnotesize$\pm$2.42}} & {\color{gray}1.01\textsubscript{{\footnotesize$\pm$0.88}}} & {\color{gray}2.48\textsubscript{{\footnotesize$\pm$1.22}}} &
{\color{gray}-0.33\textsubscript{{\footnotesize$\pm$0.53}}}\\ 

\multicolumn{1}{|c|}{} & \begin{tabular}[c]{@{}c@{}}UDJ-FL: $\beta=0.1, \gamma=0$\end{tabular} & 93.44\textsubscript{{\footnotesize$\pm$0.95}} & 88.07\textsubscript{{\footnotesize$\pm$2.16}} & 95.17\textsubscript{{\footnotesize$\pm$0.55}} & {\color{gray}4.11\textsubscript{{\footnotesize$\pm$1.18}}} & {\color{gray}0.24\textsubscript{{\footnotesize$\pm$3.69}}} & 
{\color{gray}-0.78\textsubscript{{\footnotesize$\pm$0.05}}}& 98.39\textsubscript{{\footnotesize$\pm$0.75}} & 97.23\textsubscript{{\footnotesize$\pm$1.72}} & 97.40\textsubscript{{\footnotesize$\pm$2.78}} & {\color{gray}1.33\textsubscript{{\footnotesize$\pm$0.72}}} & {\color{gray}2.65\textsubscript{{\footnotesize$\pm$1.33}}} &
{\color{gray}-0.25\textsubscript{{\footnotesize$\pm$0.49}}}\\ 

\multicolumn{1}{|c|}{} & \begin{tabular}[c]{@{}c@{}}\textbf{UDJ-FL}: $\beta=0, \gamma=-1$\end{tabular} & 92.97\textsubscript{{\footnotesize$\pm$1.43}} & 84.63\textsubscript{{\footnotesize$\pm$3.58}} & 94.83\textsubscript{{\footnotesize$\pm$0.06}} & {\color{gray}5.35\textsubscript{{\footnotesize$\pm$0.74}}} & {\color{gray}-1.47\textsubscript{{\footnotesize$\pm$4.59}}} &
{\color{gray}-0.83\textsubscript{{\footnotesize$\pm$0.24}}}& \textbf{98.61\textsubscript{{\footnotesize$\pm$0.54}}} & 97.76\textsubscript{{\footnotesize$\pm$0.92}} & 97.67\textsubscript{{\footnotesize$\pm$2.32}} & {\color{gray}0.98\textsubscript{{\footnotesize$\pm$0.66}}} & {\color{gray}2.94\textsubscript{{\footnotesize$\pm$0.28}}} & 
{\color{gray}-0.20\textsubscript{{\footnotesize$\pm$0.67}}}\\

\multicolumn{1}{|c|}{} & \begin{tabular}[c]{@{}c@{}}\textbf{UDJ-FL}: $\beta=0.1, \gamma=-1$\end{tabular} & \textbf{94.26\textsubscript{{\footnotesize$\pm$1.40}}} & 82.07\textsubscript{{\footnotesize$\pm$3.98}} & 96.17\textsubscript{{\footnotesize$\pm$0.12}} & {\color{gray}6.81\textsubscript{{\footnotesize$\pm$1.70}}} & {\color{gray}-4.98\textsubscript{{\footnotesize$\pm$4.27}}} & 
{\color{gray}-0.95\textsubscript{{\footnotesize$\pm$0.06}}}& 98.39\textsubscript{{\footnotesize$\pm$0.48}} & 96.27\textsubscript{{\footnotesize$\pm$0.75}} & 98.60\textsubscript{{\footnotesize$\pm$0.52}} & {\color{gray}1.41\textsubscript{{\footnotesize$\pm$0.00}}} & {\color{gray}1.73\textsubscript{{\footnotesize$\pm$0.76}}} &
{\color{gray}-0.83\textsubscript{{\footnotesize$\pm$0.02}}}\\ \hline

\multicolumn{1}{|c|}{\multirow{4}{*}{Egalitarian}} & TERM & 91.74\textsubscript{{\footnotesize$\pm$1.87}} & 88.13\textsubscript{{\footnotesize$\pm$2.54}}& 93.90\textsubscript{{\footnotesize$\pm$0.52}} & 2.64\textsubscript{{\footnotesize$\pm$0.73}} & {\color{gray}1.08\textsubscript{{\footnotesize$\pm$2.52}}} &
{\color{gray}-0.87\textsubscript{{\footnotesize$\pm$0.13}}}& 97.83\textsubscript{{\footnotesize$\pm$0.15}} & 93.20\textsubscript{{\footnotesize$\pm$0.89}} & 97.23\textsubscript{{\footnotesize$\pm$2.11}} & 1.96\textsubscript{{\footnotesize$\pm$0.16}} & {\color{gray}-0.40\textsubscript{{\footnotesize$\pm$0.75}}} &
{\color{gray}-0.68\textsubscript{{\footnotesize$\pm$0.41}}}\\

\multicolumn{1}{|c|}{} & FedMGDA+ & 91.51\textsubscript{{\footnotesize$\pm$1.58}} & 88.07\textsubscript{{\footnotesize$\pm$1.54}} & 94.30\textsubscript{{\footnotesize$\pm$0.87}} & 2.88\textsubscript{{\footnotesize$\pm$0.37}} & {\color{gray}1.14\textsubscript{{\footnotesize$\pm$1.56}}} & 
{\color{gray}-0.86\textsubscript{{\footnotesize$\pm$0.14}}}
&97.73\textsubscript{{\footnotesize$\pm$0.06}} & 93.63\textsubscript{{\footnotesize$\pm$1.46}} & 97.10\textsubscript{{\footnotesize$\pm$2.33}} & 1.84\textsubscript{{\footnotesize$\pm$0.51}} & {\color{gray}0.00\textsubscript{{\footnotesize$\pm$1.56}}} & 
{\color{gray}-0.61\textsubscript{{\footnotesize$\pm$0.51}}}\\

\multicolumn{1}{|c|}{} & \begin{tabular}[c]{@{}c@{}}UDJ-FL: $\beta=1, \gamma=0$\end{tabular} & 91.08\textsubscript{{\footnotesize$\pm$1.86}} & 88.23\textsubscript{{\footnotesize$\pm$2.59}} & 93.10\textsubscript{{\footnotesize$\pm$0.26}} & \underline{2.43\textsubscript{{\footnotesize$\pm$1.00}} }& {\color{gray}1.62\textsubscript{{\footnotesize$\pm$2.86}}} & 
{\color{gray}-0.84\textsubscript{{\footnotesize$\pm$0.16}}}& 98.44\textsubscript{{\footnotesize$\pm$0.86}} & 97.37\textsubscript{{\footnotesize$\pm$2.06}} & 97.80\textsubscript{{\footnotesize$\pm$2.42}} & \underline{1.01\textsubscript{{\footnotesize$\pm$0.88}}} & {\color{gray}2.48\textsubscript{{\footnotesize$\pm$1.22}}} &
{\color{gray}-0.33\textsubscript{{\footnotesize$\pm$0.53}}}\\ 

\multicolumn{1}{|c|}{} & \begin{tabular}[c]{@{}c@{}}\textbf{UDJ-FL}: $\beta=1, \gamma=1$\end{tabular} & 92.37\textsubscript{{\footnotesize$\pm$0.91}} & 90.10\textsubscript{{\footnotesize$\pm$0.10}} & 93.70\textsubscript{{\footnotesize$\pm$0.95}} & \textbf{2.21\textsubscript{{\footnotesize$\pm$0.28}}}& {\color{gray}3.81\textsubscript{{\footnotesize$\pm$0.29}}} & 
{\color{gray}-0.61\textsubscript{{\footnotesize$\pm$0.22}}}& 98.00\textsubscript{{\footnotesize$\pm$0.83}} & 97.93\textsubscript{{\footnotesize$\pm$1.84}} & 97.63\textsubscript{{\footnotesize$\pm$1.99}} & \textbf{0.94\textsubscript{{\footnotesize$\pm$0.45}}} & {\color{gray}3.15\textsubscript{{\footnotesize$\pm$1.22}}} &
{\color{gray}0.31\textsubscript{{\footnotesize$\pm$0.42}}}\\ \hline

\multicolumn{1}{|c|}{\multirow{5}{*}{\begin{tabular}[c]{@{}c@{}}Rawls'\\ Difference\\ Principle\end{tabular}}} & PropFair: $M=2$ & 90.51\textsubscript{{\footnotesize$\pm$1.47}} & 88.17\textsubscript{{\footnotesize$\pm$1.02}} & 92.57\textsubscript{{\footnotesize$\pm$0.86}} & {\color{gray}2.62\textsubscript{{\footnotesize$\pm$0.53}}} & 2.30\textsubscript{{\footnotesize$\pm$1.53}} & 
{\color{gray}-0.82\textsubscript{{\footnotesize$\pm$0.21}}}& 97.93\textsubscript{{\footnotesize$\pm$1.10}} & 92.36\textsubscript{{\footnotesize$\pm$2.94}} & 96.80\textsubscript{{\footnotesize$\pm$1.95}} & {\color{gray}2.25\textsubscript{{\footnotesize$\pm$1.38}}}  & -0.72\textsubscript{{\footnotesize$\pm$4.19}} & 
{\color{gray}-0.56\textsubscript{{\footnotesize$\pm$0.21}}}\\

\multicolumn{1}{|c|}{} & AFL & 91.04\textsubscript{{\footnotesize$\pm$1.44}} & 91.20\textsubscript{{\footnotesize$\pm$0.78}} & 91.27\textsubscript{{\footnotesize$\pm$0.97}} & {\color{gray}2.44\textsubscript{{\footnotesize$\pm$0.88}}} & \textbf{7.07\textsubscript{{\footnotesize$\pm$1.03}}} & 
{\color{gray}-0.09\textsubscript{{\footnotesize$\pm$0.23}}}& 96.97\textsubscript{{\footnotesize$\pm$0.75}} & 95.70\textsubscript{{\footnotesize$\pm$0.52}} & 95.03\textsubscript{{\footnotesize$\pm$1.25}} & {\color{gray}1.15\textsubscript{{\footnotesize$\pm$0.29}}} & 3.41\textsubscript{{\footnotesize$\pm$1.03}}  & 
{\color{gray}0.48\textsubscript{{\footnotesize$\pm$0.38}}}\\


\multicolumn{1}{|c|}{} & \begin{tabular}[c]{@{}c@{}}$q$-FedAvg: $q$=5\end{tabular} & 91.32\textsubscript{{\footnotesize$\pm$1.05}} & 85.80\textsubscript{{\footnotesize$\pm$0.78}} & 90.70\textsubscript{{\footnotesize$\pm$1.37}} & {\color{gray}3.38\textsubscript{{\footnotesize$\pm$0.84}}} & 4.68\textsubscript{{\footnotesize$\pm$0.72}} & 
{\color{gray}-0.45\textsubscript{{\footnotesize$\pm$0.10}}}& 87.10\textsubscript{{\footnotesize$\pm$2.75}} & 86.93\textsubscript{{\footnotesize$\pm$1.66}} & 77.23\textsubscript{{\footnotesize$\pm$4.45}} & {\color{gray}4.40\textsubscript{{\footnotesize$\pm$1.31}}} & 6.50\textsubscript{{\footnotesize$\pm$2.94}} &
{\color{gray}0.67\textsubscript{{\footnotesize$\pm$0.10}}}\\


\multicolumn{1}{|c|}{} & \begin{tabular}[c]{@{}c@{}}UDJ-FL: $\beta=5, \gamma=0$\end{tabular} & 91.41\textsubscript{{\footnotesize$\pm$0.88}} & 85.87\textsubscript{{\footnotesize$\pm$0.31}} & 91.10\textsubscript{{\footnotesize$\pm$1.40}} & {\color{gray}3.63\textsubscript{{\footnotesize$\pm$0.88}}} & 4.46\textsubscript{{\footnotesize$\pm$1.05}} & 
{\color{gray}-0.51\textsubscript{{\footnotesize$\pm$0.23}}}& 86.61\textsubscript{{\footnotesize$\pm$2.61}} & 87.37\textsubscript{{\footnotesize$\pm$2.82}} & 77.93\textsubscript{{\footnotesize$\pm$2.89}} & {\color{gray}4.11\textsubscript{{\footnotesize$\pm$1.37}}} & \underline{7.20\textsubscript{{\footnotesize$\pm$4.84}}} &  
{\color{gray}0.73\textsubscript{{\footnotesize$\pm$0.13}}}\\


\multicolumn{1}{|c|}{} & \begin{tabular}[c]{@{}c@{}}\textbf{UDJ-FL}: $\beta=5, \gamma=1$\end{tabular} & 90.42\textsubscript{{\footnotesize$\pm$1.23}} & 86.60\textsubscript{{\footnotesize$\pm$0.70}} & 90.00\textsubscript{{\footnotesize$\pm$1.37}} & {\color{gray}3.55\textsubscript{{\footnotesize$\pm$0.85}}} & \underline{6.04\textsubscript{{\footnotesize$\pm$1.72}}}& 
{\color{gray}-0.23\textsubscript{{\footnotesize$\pm$0.17}}}
& 85.33\textsubscript{{\footnotesize$\pm$2.68}} & 87.50\textsubscript{{\footnotesize$\pm$2.98}} & 75.13\textsubscript{{\footnotesize$\pm$3.15}} & {\color{gray}5.40\textsubscript{{\footnotesize$\pm$1.02}}} & \textbf{8.29\textsubscript{{\footnotesize$\pm$4.78}}} & 
{\color{gray}0.79\textsubscript{{\footnotesize$\pm$0.06}}}\\ \hline

\multicolumn{1}{|c|}{\multirow{3}{*}{Desert}} & CFFL & 87.92\textsubscript{{\footnotesize$\pm$1.28}} & 84.97\textsubscript{{\footnotesize$\pm$0.97}} & 95.61\textsubscript{{\footnotesize$\pm$0.63}} & {\color{gray}6.12\textsubscript{{\footnotesize$\pm$2.02}}} & {\color{gray}0.79\textsubscript{{\footnotesize$\pm$0.58}}} & -0.83\textsubscript{{\footnotesize$\pm$0.58}} & 93.75\textsubscript{{\footnotesize$\pm$1.18}} & 92.92\textsubscript{{\footnotesize$\pm$1.25}}& 96.67\textsubscript{{\footnotesize$\pm$2.32}}& {\color{gray}3.19\textsubscript{{\footnotesize$\pm$0.92}}}& {\color{gray}3.43\textsubscript{{\footnotesize$\pm$1.68}}} & -0.32\textsubscript{{\footnotesize$\pm$0.67}} \\

\multicolumn{1}{|c|}{} & \begin{tabular}[c]{@{}c@{}} UDJ-FL: $\beta=\bm{\beta}$, $\gamma=-1$\end{tabular} & 83.11\textsubscript{{\footnotesize$\pm$9.01}} & 58.60\textsubscript{{\footnotesize$\pm$5.31}} & 96.53\textsubscript{{\footnotesize$\pm$1.00}} & {\color{gray}16.32\textsubscript{{\footnotesize$\pm$0.73}}} & {\color{gray}-12.49\textsubscript{{\footnotesize$\pm$1.64}}} & \underline{-0.88\textsubscript{{\footnotesize$\pm$0.08}}}& 86.56\textsubscript{{\footnotesize$\pm$5.80}}& 41.23\textsubscript{{\footnotesize$\pm$4.21}} & 98.90\textsubscript{{\footnotesize$\pm$0.52}} & {\color{gray}23.76\textsubscript{{\footnotesize$\pm$1.49}}} & {\color{gray}-28.59\textsubscript{{\footnotesize$\pm$8.64}}} &  \underline{-0.94\textsubscript{{\footnotesize$\pm$0.03}}}  \\ 

\multicolumn{1}{|c|}{} & \begin{tabular}[c]{@{}c@{}}\textbf{UDJ-FL}: $\beta=\bm{\beta}$, $\gamma=0$\end{tabular} & 88.21\textsubscript{{\footnotesize$\pm$4.36}} & 64.70\textsubscript{{\footnotesize$\pm$3.55}} & 95.60\textsubscript{{\footnotesize$\pm$0.82}} & {\color{gray}13.30\textsubscript{{\footnotesize$\pm$0.31}}} & {\color{gray}-11.98\textsubscript{{\footnotesize$\pm$3.06}}} & \textbf{-0.93\textsubscript{{\footnotesize$\pm$0.03}}}& 89.44\textsubscript{{\footnotesize$\pm$4.17}}& 46.27\textsubscript{{\footnotesize$\pm$6.65}} & 98.33\textsubscript{{\footnotesize$\pm$0.45}} & {\color{gray}22.16\textsubscript{{\footnotesize$\pm$1.07}}} & {\color{gray}-28.00\textsubscript{{\footnotesize$\pm$7.31}}} &  \textbf{-0.97\textsubscript{{\footnotesize$\pm$0.02}}}  \\ \hline

\end{tabular}%
}
\end{table}

\subsection{Results}
We lists the results of all experiments for both DirtyMNIST and CURE-TSR in Table \ref{tab:results}. The best overall results per section are listed in bold while the second best are underlined. Results written in gray are not relevant to the particular fairness type being analyzed, but are listed simply for completeness and comparison with other fairness types. 

\textbf{Utilitarian:} For utilitarian fairness, we are most concerned with obtaining the highest overall global model accuracy (first column in Table \ref{tab:results}, marked Global Acc). For both DirtyMNIST and CURE-TSR, UDJ-FL was able to obtain the highest overall global model accuracy with $\beta\to0$ and $\gamma=-1$ as hypothesized in Section \ref{sec:method} to achieve utilitarian fairness. However, for DirtyMNIST, UDJ-FL obtains higher accuracy with $\beta=0.1$ than $\beta=0$. We believe this is caused by $\beta=0$ (combined with $\gamma=-1$) not giving the clients with high aleatoric uncertainty enough attention by the global model and therefore not properly learning from their data distributions while $\beta=0.1$ does. To find the best overall utility using UDJ-FL (and specifically Eq. \ref{eq:util-opt}), we suggest testing several values in the range of $\beta\in(0, 1)$. 

\textbf{Egalitarian:} We compare egalitarian fairness methods along their achieved standard deviation of the clients' test accuracies (column labeled STD in Table \ref{tab:results}). For both DirtyMNIST and CURE-TSR, UDJ-FL achieved the lowest standard deviation among the clients' final test accuracies with $\beta=1$ and $\gamma=1$ which validates that this hyperparameter setting allows UDJ-FL to achieve egalitarian fairness. Additionally, we note that while the objective function for Rawls' DP and egalitarian fairness are similar, the standard deviations obtained by the Rawls' DP fairness methods are higher than those obtained by UDJ-FL with $\beta=1$ and $\gamma=1$ further supporting the claim that Rawls' DP approaches, while focusing attention on the least advantaged user, does not inherently cause the standard deviation among the client accuracies to be smaller. 

\textbf{Rawls' Difference Principle:} To analyze the effectiveness of UDJ-FL and the other Rawls' DP fairness methods, we look at the absolute accuracy difference between the least advantaged client and average difference of all other clients (column titled $\Psi$ Eq. \ref{eq:Psi}). For DirtyMNIST, UDJ-FL with $r=1+\frac{1}{\beta}, \beta\to\infty$, and $\gamma=1$ achieve the second highest overall $\Psi$ score, while for CURE-TSR achieved the highest, meaning that for both datasets, UDJ-FL training benefited the least advantaged client the most which meets the requirement for Rawls' difference principle fairness. 

\textbf{Desert:} In desert-based fairness, we are most concerned with the final accuracy distribution along the clients local datasets being inversely proportional with the distribution of the clients' aleatoric uncertainty values. (column labeled $r_{\bm{\upsilon},\bm{A}}$ Eq. \ref{eq:pearson} in Table \ref{tab:results}). For both DirtyMNIST and CURE-TSR, our UDJ-FL method with $r=1$, $\beta<0$, and $\gamma=0$ achieve the first and second best results (most negative $r_{\bm{\upsilon},\bm{A}}$ value) when compared against the other desert fair methods meaning that UDJ-FL successfully achieves desert fairness. 

\subsection{Limitations} While UDJ-FL is able to achieve all four notions of distributive justice fairness through simple hyperparameter selection, it does have a few limitations. First, Lemma \ref{lem} provides an upper bound on the local Lipschitz constant for the client-level loss gradient. However, in certain cases, this approximation can lead to a model that does not achieve the best fairness-accuracy tradeoff. Second, while we have left it up to the practitioner to best choose which fairness notion to use in various scenarios, we acknowledge that improperly selecting the fairness metric can lead to issues. For instance, using utilitarian fairness even in the presence of severe across-client data quality disparities would result in a global model that only works well for a small subset of the clients. Finally, we make the assumption that aleatoric uncertainty accurately captures client contribution and advantage. However, this assumption may not hold in all cases (e.g., highly noisy or adversarial data distributions).

\section{Conclusion}
\label{sec:conclusion}
In this work, we presented Uncertainty-based Distributive Justice for Federated Learning (UDJ-FL) which is a flexible federated learning framework that can achieve four main theories of distributive justice -- egalitarian, utilitarian, Rawls' difference principle, and desert -- by utilizing techniques from fair resource allocation in conjunction with using the aleatoric uncertainty score of the clients to define who is the least advantaged and/or who contributes the most to the overall federated learning process. We empirically show that our UDJ-FL method is able to successfully achieve all four definitions of fairness through hyperparameter selection which makes it simple for practitioners to change the implemented fairness guarantee without having to make heavy architectural changes. In the future, we plan to extend UDJ-FL into more complex federated settings such as cross-device where clients can be unreliable, and consider more privacy preserving settings such as implementing a differentially private version of UDJ-FL to further protect the clients' local data privacy. 

\bibliographystyle{ACM-Reference-Format}
\bibliography{ref}

\appendix
\section{Motivating Example}
To further motivate our reasoning behind using uncertainty quantification, specifically aleatoric uncertainty, based reweighing we offer the following example. In normal federated averaging (FedAvg), the weighing strategy gives more importance to clients that have more data points. But it could be the case that the client with the most amount of data points has noisy or bad quality data. In this setting, it is likely that the global model will over-fit to this client’s dirty data the federated model's performance will suffer. To show this empirically, we trained a simple federated learning model using the Dirty-MNIST dataset (see Section \ref{sec:data} for a description) and 5 clients in three different settings: 1) all clients had mostly clean data with only a small amount of dirty data; 2) four of the clients had mostly dirty data while the remaining client had an overwhelming number of clean data points; and 3) four of the clients has mostly clean data while the remaining had an overwhelming number of dirty data points. We report the global model accuracy, along with each client's local dataset accuracy in Table \ref{tab:even-clean-large}. Here, due to using different data splits, we cannot necessarily compare the client accuracies between each setting. However, we can compare each client's accuracy against the other clients within each setting as well as compare the global model performance. For instance, in the first setting (first row of Table \ref{tab:even-clean-large} labeled `Even'), the average accuracy along all the clients is fairly stable and the global model accuracy is acceptable\footnote{We note that the global model accuracy is less than any of the client's local performance, which may seem odd. However, we chose to test on the full default test set provided with MNIST rather than on the combination of each client's test sets (which include both clean MNIST and dirty Ambiguous-MNIST images) to test generalizability of the global model to unseen clean test points.}. In the second setting (second row of Table \ref{tab:even-clean-large} labeled `Clean'), the accuracy of the first client -- who had the overwhelming amount of clean data -- had performance greater than any of the other four clients who had mostly dirty data. However, the global model performance increased from the previous setting (even though the amount of dirty data present in the overall system increased) since the client with the large amount of clean data was given more weight during model averaging. Finally, in the third setting (third row of Table \ref{tab:even-clean-large} labeled `Dirty') while the fifth client, who had an overwhelming amount of dirty data, did not achieve the highest performance on their local data distribution (which is reasonable due to having mostly dirty data while the other clients had mostly clean data), the global model achieved lower performance than the other two settings, even though the amount of clean data present in the system was similar to the second setting. This was due to the model giving more importance to the fifth client who had primarily dirty data. 
\begin{table}[h!]
\centering
\caption{Client accuracy under different data distribution settings in FedAvg.}
\label{tab:even-clean-large}
\begin{tabular}{@{}cccccccc@{}}
\cmidrule(l){2-8}
 & \multicolumn{6}{c}{Accuracy} & \multirow{2}{*}{\begin{tabular}[c]{@{}c@{}}Client \\STD\end{tabular}} \\ \cmidrule(lr){2-7}
 & 1 & 2 & 3 & 4 & 5 & Global & \\ \midrule
Even & 94.56 & 91.50 & 94.47 & 95.2 & 94.87 & 87.74 & 1.49\\
Clean & 94.50 & 82.40 & 88.67 & 85.43 & 79.10 & 89.83 & 5.92 \\
Dirty & 88.33 & 85.13 & 92.33 & 92.33 & 88.13 & 83.55 & 3.06\\ \bottomrule
\end{tabular}
\end{table}

\section{Definitions for Distributive Justice}
\label{app:dj-definitions}
\begin{table}[h!]
\centering
\caption{Four common distributive justice principles, their description, and example as it relates to federated learning.}
\label{tab:def-and-examples}
\resizebox{.99\textwidth}{!}{%
\begin{tabular}{@{}ccc@{}}
\toprule
\textbf{Type} & \textbf{Description} & \textbf{FL Example} \\ \midrule
Egalitarian & \mline{Every person should receive and equal amount of material goods \& services.} & \mline{Clients achieve similar final model performance on their local data distributions.}  \\ \midrule
Desert & \mline{Everyone should be allocated material goods and services based on the value of their contribution/effort/cost.} & \mline{Clients receive final model performance on their local data distribution relative to the quality of their data.} \\ \midrule
\mline{Rawls' Difference Principle} & \mline{Only material difference allowed are those that raise the level of the least advantaged. Concerned with absolute increase, not relative, and not concerned with maximizing utility.} & \mline{Maximize the performance of the worst performing client/client with lowest quality data.} \\ \midrule
Utilitarianism & \mline{Distribute material goods \& services such that overall well-being is maximized.} & \mline{Maximize the average performance, e.g., standard FedAvg.} \\ \bottomrule
\end{tabular}%
}
\end{table}
Here, we provide definitions for the four distributive justice theories supported by UDJ-FL. We note that other distributive justice theories exist beyond the four discussed here. However, they are difficult to realize in the federated setting. For example, in luck-egalitarian equality of opportunity there is no clear way to measure ambitions  (our choices and what results from them, such as the choice to work hard), or endowments (results of brute luck, or those things over which we have no control, such as one's genetic inheritance), and there is no direct way to disentangle them in the federated learning scenario. Additionally, we do not discuss libertarianism in this work as there is no clear mapping of libertarianism (which defines the concept of just acquisition and exchanging of material goods and services compared to the simple distribution of them) to the goal of federated learning -- which is to maximize the utility of the clients. We leave a fair federated learning framework that covers these types of distributive justice for future work.

\subsection{Strict Egalitarianism} Often thought of as the simplest form of distributive justice, strict egalitarianism can be defined as follows:
\begin{definition}[Strict Egalitarianism \cite{sep-justice-distributive}]
    All persons in a society should have the same level of material goods and services as all humans are morally equal. 
\end{definition}
The strict egalitarian principle is most commonly justified on the grounds that people are morally equal and that equality in material goods and services is the best way to give effect to this moral ideal. While strict egalitarianism is simple to define, there are many critiques of the theory. For instance, the definition of strict egalitarianism implies that even if an unequal distribution would make everyone in the society better off, or if an unequal distribution would make some in society better off with no one being worse off, the strict egalitarian principle of equal distribution should still be maintained \cite{sep-justice-distributive}. Further, in federated learning, it is infeasible to aim for \textit{strict} egalitarianism (where all client achieve the exact same final performance) due to the probabilistic nature of machine learning and therefore we relax strict egalitarianism to simple egalitarianism in this work.

\subsection{Utilitarianism} Utilitarianism is a welfare-based principle and is primarily concerned with the overall level of ``welfare'' of a population where welfare is often defined in term of pleasure, happiness, or preference-satisfaction \cite{sep-justice-distributive}. It can be defined as follows:
\begin{definition}[Utilitarianism]
    The distribution of material goods and services should maximize the expected utility of pleasure/happiness/preference-satisfaction within the population. 
\end{definition}
Here, preference-satisfaction refers to individual persons in the society obtaining what they desire. 
Utilitarianism is often criticized since it does not take into consideration the distinctness of individuals within a society. 

\subsection{Rawls' Difference Principle} First proposed by John Rawls in \cite{rawls1971atheory}, Rawls' difference principle can be stated as follows: 
\begin{definition}[Rawls' Difference Principle \cite{sep-justice-distributive}]
    Social and economic inequalities are to be to the greatest benefit of the least advantaged members of society.
\end{definition}
In other words, the only allowable difference in material distribution are those that help the least advantaged in society. Rawls' difference principle is one of the most discussed, and therefore most criticized, distributive justice theories over the past four decades and we refer interested readers to \cite{sep-justice-distributive} for an in-depth discussion.

\subsection{Desert}
Desert is a normative concept that occurs often in our daily life. E.g., ``\textit{you get what you deserve}''. There are many thought principles for desert fairness which differ according to how ``deserving'' is defined. In this work, we use the following definition of desert fairness:
\begin{definition}[Desert \cite{sep-justice-distributive}]
    People should be rewarded for their work activity according to the \underline{value} of their contribution to the social product.
\end{definition}
In other words, people get rewards commensurate to the amount of quality work they contribute. Those that contribute more quality work to a process (e.g., contribute more high quality data), should receive better rewards than those who contribute low quality work or those who are free-riders.

\section{Choosing Aleatoric Uncertainty as Client Weights}
\label{sec:irreducible-bayes}
Here, we provide intuition for why we choose to use aleatoric uncertainty as our client weights (rather than using epistemic or predictive uncertainty) in our UDJ-FL framework. One main reason we choose to use aleatoric uncertainty is that it is irreducible while epistemic uncertainty and predictive uncertainty can be improved with training. When choosing how to define ``least advantaged client'' or ``contribution'' we want to use an easy to calculate, steady, value that remains the same throughout training. More specifically, given a loss function $\ell(y, \hat{y})$, the \textbf{point-wise risk} (or \textit{expected loss}) of a predictor $f$ at $\bm{x}\in\mathcal{X}$ is defined as 
\begin{equation}
    \label{eq:point-wise-risk}
    R(f,\bm{x}) = \mathbb{E}_{\mathbb{P}(Y\mid X=\bm{x})}[\ell(Y, f(\bm{x})]
\end{equation}
Eq. \ref{eq:point-wise-risk} has a fundamental limiting lower bound, usually reached at a function $f^*$  called the \textbf{Bayes Predictor}:
\begin{equation}
    \label{eq:bayes-predictor}
    f^*(\bm{x}) = \underset{\hat{y}\in\mathcal{Y}}{\mathrm{argmin}}\mathbb{E}_{\mathbb{P}(Y\mid X=\bm{x})}[\ell(Y, \hat{y})]
\end{equation}
In \cite{lahlou2021deup}, the authors show that $R(f^*, \bm{x}) > 0$ indicates an irreducible risk due to the inherent randomness of $\mathbb{P}(Y\mid X=\bm{x})$ and therefore can be used as a measure of aleatoric uncertainty at $\bm{x}$. Therefore, a client in the federation that has higher average aleatoric uncertainty will have higher training loss due to having a higher fundamental limiting lower bound. Since FedAvg reweighs clients according to dataset size, if the largest client has high aleatoric uncertainty, then their model parameters are sub-optimal and will degrade the overall performance of the federated learning model as well as the performance of the model on other client's data distributions. In turn, if we use aleatoric uncertainty as our weights, we can selectively give clients with the least advantage (e.g., higher limiting lower bound) or the highest amount of contribution (e.g., lowest limiting bound) more attention. 

In addition to aleatoric uncertainty being irreducible, epistemic and predictive uncertainty require more advanced techniques to estimate such as Bayesian modeling, Monte- Carlo drop-out, or model ensembling \cite{abdar2021review}. Aleatoric uncertainty on the other hand can be simply calculated using outputs from the softmax layer of the model. Further, since all clients are ultimately training one model together, the epistemic uncertainty of each client should converge by the end of training, meaning it is not a useful measure of ``advantage'' or ``contribution''.

\section{Generalization Bounds of UDJ-FL}
UDJ-FL can be seen as a generalization of $q$-FFL in that UDJ-FL recovers $q$-FFL when $r=1+\frac{1}{\beta}$ and $\beta>0$ in addition to achieving the other three distributive justice fairness measures of egalitarianism, desert, and utilitarianism. $q$-FFL itself is a generalization of AFL \cite{mohri2019agnostic} as $q\to 0$. Similar to the generalization bounds provided in \cite{li2019fair}, we start from the AFL objective. 

\textit{AFL objective}: Suppose we have a federated setting in which we want to minimize the loss over the clients but the proper weights for each client are unknown:
\begin{equation}
\label{eq:optimal-afl}
    L_\lambda(h) = \sum_{i=1}^N\lambda_i\mathbb{E}_{(x,y)\sim D_i}[\ell(h(x), y)]
\end{equation}
where $\lambda\in\Lambda$, $\Lambda$ is a probability simplex, and $D_i$ is client $i$'s local dataset. Denote by $\hat{L}_\lambda(h)$ the empirical loss:
\begin{equation}
\label{eq:empirical-afl}
    \hat{L}_\lambda(h) = \sum_{i=1}^N\frac{\lambda_i}{n_i}\sum_{j=1}^{n_i}\ell(h(x_{i,j}), y_{i,j})
\end{equation}
In AFL \cite{mohri2019agnostic}, the goal is to derive $\lambda$ (the client weights) such that the most disadvantaged client receives the biggest benefit from federated training (e.g., Rawls' difference principle fairness). However, in UDJ-FL, we assume $\lambda$ is set to be the aleatoric uncertainty values of the clients, i.e., $\lambda=\bm{\upsilon}$. Therefore, we can rewrite $L_\lambda(h)$ (Eq. \ref{eq:optimal-afl}) and $\hat{L}_\lambda(h)$ (Eq. \ref{eq:empirical-afl}) as:
\begin{equation}
\label{eq:upsilon-optimal}
        L_{\bm{\upsilon}}(h) = \sum_{i=1}^N\upsilon_i\mathbb{E}_{(x,y)\sim D_i}[\ell(h(x), y)]
\end{equation}
\begin{equation}
\label{eq:upsilon-empirical}
            \hat{L}_{\bm{\upsilon}}(h) = \sum_{i=1}^N\frac{\upsilon_i}{n_i}\sum_{j=1}^{n_i}\ell(h(x_{i,j}), y_{i,j})
\end{equation}
We also consider a unweighted version of UDJ-FL:
\begin{equation}
    \label{eq:distjust-as-afl-optimal}
    L_{r\beta}(h) = \sum_{i=1}^N\left(\mathbb{E}_{(x,y)\sim D_i}[\ell(h(x), y)]\right)^{r\beta}
\end{equation}
and further, construct the empirical loss of the unweighted UDJ-FL similar to Eq. \ref{eq:upsilon-empirical} as:
\begin{equation}
    \label{eq:distjust-as-afl-empirical}
      \hat{L}_{r\beta}(h) =\sum_{i=1}^N\frac{1}{n_i}\left[\sum_{j=1}^{n_i}\ell(h(x_{i,j}), y_{i,j})\right]^{r\beta}
\end{equation}
Eq. \ref{eq:distjust-as-afl-empirical} can be written as:
\begin{equation}
    \label{eq:distjust-max}
    \Tilde{L}_{r\beta}(h) = \max_{\bm{\sigma}, ||\bm{\sigma}||_p\leq1}\sum_{i=1}^N\frac{\sigma_i}{n_i}\sum_{j=1}^{n_i}\ell(h(x_{i,j}), y_{i,j})
\end{equation}
where $\frac{1}{p} + \frac{1}{r\beta} = 1$, $(p\geq1, r\beta\geq0)$. Here, the goal is to obtain $\frac{\sigma_i}{n_i}\sum_{j=1}^{n_i}\ell(h(x_{i,j}), y_{i,j})$ as large as possible by choosing the proper $\bm{\sigma}$. By giving more weight (higher $\sigma_i$) to clients with higher losses we can simulate raising the client's losses to $r\beta$. 

\begin{lemma}[Generalization bounds of UDJ-FL for a specific $\lambda=\bm{\upsilon}$]
    Assume that the loss $l$ is bounded by $M > 0$ and the numbers of local samples are $(n_1, \dots, n_N)$. Then, for any $\delta > 0$, with probability at least $1-\delta$, the following holds for any $\bm{\upsilon}\in\Upsilon$, $h\in H$:
    \begin{equation}
        L_{\bm{\upsilon}}(h) \leq ||\bm{\upsilon}||_p\Tilde{L}_{r\beta}(h) + \mathbb{E}\left[\max_{h\in\mathcal{H}}L_{\bm{\upsilon}}(h) - \hat{L}_{\bm{\upsilon}}(h)\right]+ M\sqrt{\sum_{i=1}^N\frac{v_i^2}{2n_i}\log\frac{1}{\delta}}
    \end{equation}
\end{lemma}

\begin{proof}
    As shown in \cite{mohri2019agnostic}, for any $\delta>0$, the following inequality holds with probability at least $1-\delta$ for any $\bm{\upsilon}\in\Upsilon$, $h\in\mathcal{H}$:
    \begin{equation}
    \label{eq:proof}
        L_{\bm{\upsilon}}(h) \leq \hat{L}_{\bm{\upsilon}}(h)+ \mathbb{E}\left[\max_{h\in\mathcal{H}}L_{\bm{\upsilon}}(h) - \hat{L}_{\bm{\upsilon}}(h)\right]+ M\sqrt{\sum_{i=1}^N\frac{v_i^2}{2n_i}\log\frac{1}{\delta}}
    \end{equation}
    Denote the empirical loss on device $i$ as $\frac{1}{n_i}\sum_{j=1}^{n_i}\ell(h(x_{i,j}), y_{i,j})$ as $F_i$. From H{\"o}lder's inequality, we have:
    \begin{equation}
        \hat{L}_{\bm{\upsilon}}(h) = \sum_{i=1}^N\upsilon_iF_i \leq \left(\sum_{i=1}^N\upsilon_i^p\right)^{\frac{1}{p}}\left(\sum_{i=1}^N F_i^{r\beta}\right)^\frac{1}{r\beta} = ||\bm{\upsilon}||_p\Tilde{L}_{r\beta}(h),\;\;\;\;\frac{1}{p}+\frac{1}{r\beta}=1
    \end{equation}
    Plugging $\hat{L}_{\bm{\upsilon}}(h) \leq ||\bm{\upsilon}||_p\Tilde{L}_{r\beta}(h)$ into Eq. \ref{eq:proof} yields the results.
\end{proof}

\section{Additional Experimental Settings}
\label{app:experimental}
 We compare UDJ-FL against multiple baselines that span all four of the distributive justice types. For utilitarian, we test against PropFair \cite{zhang2022proportional} and $q$-FFL \cite{li2019fair}. For egalitarian, we test against TERM \cite{li2020tilted} and FedMGDA+ \cite{hu2020fedmgda+}. For Rawls' difference principle, we test against PropFair \cite{zhang2022proportional}, AFL \cite{mohri2019agnostic}, and $q$-FFL \cite{li2019fair}. And for desert, we test against CFFL \cite{lyu2020collaborative}. In Table \ref{tab:hyperparam-lists} we list the tested hyperparameters for each baseline as well as the chosen hyperparameter for each dataset. In general, we tested the same hyperparameter ranges as \cite{zhang2022proportional}, or those defined in the individual papers. For learning rates, we tested $\eta=\{0.1, 0.01, 0.001\}$ and found that in all settings $\eta=0.1$ achieved the best overall results. In certain cases, such as PropFair and $q$-FFL, multiple hyperparameters were chosen as different fairness values were obtained under different hyperparameter settings. We detail which hyperparameter achieved each setting in our main results shown in Table \ref{tab:results}. In all cases, we chose the hyperparameters that obtained the best fairness results, not those that achieved the best accuracy (except for utilitarian fairness).
 
\begin{table}[h!]
\centering
\caption{Tested and chosen hyperparameters for each baseline.}
\label{tab:hyperparam-lists}
\begin{tabular}{cc|cc|}
\cline{3-4}
 &  & \multicolumn{2}{c|}{Chosen Hyperparameter} \\ \hline
\multicolumn{1}{|c|}{Baseline} & Range & \multicolumn{1}{c|}{Dirty-MNIST} & CURE-TSR \\ \hline
\multicolumn{1}{|c|}{PropFair} & $M=\{2,3,4,5\}$ & $M=\{2,3\}$ & $M=\{2,3\}$ \\
\multicolumn{1}{|c|}{$q$-FFL} & $q=\{0, 0.1, 1, 2, 5\}$ & $q=\{0, 0.1, 5\}$ & $q=\{0, 0.1, 5\}$ \\ 
\multicolumn{1}{|c|}{TERM} & $\alpha=\{0.01, 0.1, 0.5\}$ & $\alpha=0.01$ & $\alpha=0.01$ \\ 
\multicolumn{1}{|c|}{FedMGDA+} & $\epsilon=\{0.05, 0.1, 0.5\}$  & $\epsilon=0.5$ & $\epsilon=0.5$ \\ 
\multicolumn{1}{|c|}{CFFL} & $\alpha=\{1,2,3,4,5\}$ & $\alpha=2$ & $\alpha=2$ \\ 
\hline
\end{tabular}%
\end{table}

\end{document}